\documentclass[lettersize,journal]{IEEEtran}

\usepackage{amsmath,amsfonts}
\usepackage{algorithmic}
\usepackage{algorithm}
\usepackage{array}
\usepackage[caption=false,font=normalsize,labelfont=sf,textfont=sf]{subfig}
\usepackage{textcomp}
\usepackage{stfloats}
\usepackage{url}
\usepackage{graphicx}
\usepackage{verbatim}
\usepackage{graphicx}
\usepackage{cite}

\usepackage{amssymb}
\usepackage{amsthm}
\DeclareMathOperator{\clip}{clip} 
\usepackage{booktabs,multirow}
\theoremstyle{plain}
\newtheorem{theorem}{Theorem}

\begin{document}

\title{Reducing ANN-SNN Conversion Error via Residual Membrane Potential Alignment}

\author{Zirui Chen, Zihan Huang, Tong Bu, Jianhao Ding, Yiting Dong,  Zhaofei Yu

\thanks{Z. Chen, T. Bu and Z. Yu are with the Institute for Artificial Intelligence, Peking University, and also with Beijing Key Laboratory of Brain-inspired Spiking Large Models, the School of Computer Science, Peking University, Beijing, China. E-mail:\{2200012913, putong30, yuzf12\}@pku.edu.cn }

\thanks{Z. Huang, J. Ding and Y. Dong are with the School of Computer Science, Peking University, and also with the National Key Laboratory for Multimedia Information Processing, Peking University, Beijing, China. E-mail: \{hzh, djh02998, dongyiting\}@pku.edu.cn}

}

\markboth{IEEE Transactions on Neural Networks and Learning Systems}%
{Chen \MakeLowercase{\textit{et al.}}: Reducing ANN-SNN Conversion Error via Residual Membrane Potential Alignment}


\maketitle

\begin{abstract}
Spiking Neural Networks (SNNs) serve as core architectures for neuromorphic computing thanks to event-driven operation and ultra-low power consumption. Direct SNN training is hindered by non-differentiable spikes that induce vanishing gradients and unstable optimization. ANN-SNN conversion circumvents such issues by reusing well-trained ANN weights for low-latency, energy-efficient inference. Nevertheless, existing conversion schemes suffer from severe accuracy drops at small timesteps, large inference delays and cumulative quantization errors, even with marginal performance loss at large $T$.
To address these limitations, we first analyze flaws of conventional conversion pipelines from residual membrane potential statistics and propose a novel conversion strategy combining dynamic initial potential tuning and feature enhancement. We then introduce a regularization loss $\mathcal{L}_{\mathrm{RMPD}}$ to adapt initial potential of IF neurons and mitigate systematic truncation bias from boundary aggregation. A dedicated SCR-Conv2d competitive refinement layer with grouped convolution is further built to sharpen feature discrimination, eliminate redundant spikes and stabilize encoding under tiny time windows. Integrated with the state-of-the-art QCFS baseline, our approach delivers consistent low-latency performance gains and generalizes to ReLU CNNs, ANN Transformers, and multi-threshold SNN variants. Evaluations on CIFAR-10, CIFAR-100 and ImageNet verify prominent accuracy improvements at $T=2,4,8$, with negligible extra computation overhead. This work offers an effective conversion paradigm to facilitate real-world SNN deployment on neuromorphic chips.
\end{abstract}


\section{Introduction}
\IEEEPARstart{S}{piking} Neural Networks (SNNs) are spike-driven biologically inspired neural networks\cite{maass1997networks}. Unlike conventional Artificial Neural Networks (ANNs), SNNs introduce discrete inference time steps, where each neuron accumulates membrane potential at every time \cite{bohte2000spikeprop,gerstner2014neuronal}; once the membrane potential exceeds a predefined firing threshold, the neuron emits a fixed spike\cite{roy2019towards,mueller2021spiking}. This discrete spike-based information transmission mechanism\cite{perez2021sparse,pei2019towards,yao2024spike,fang2020spike,zenke2021brain,cheng2026brain,YU2016,yu2013rapid} delivers remarkable energy efficiency gains for two core reasons: first, SNNs maintain drastically fewer active neurons than ANNs, with most irrelevant neurons retaining zero output throughout inference; second, fixed spike sequences generated by IF layers substantially reduce floating-point multiplication and division operations. For the above reasons, converting well-trained ANNs into equivalent SNNs stands as a vital research direction.

Theoretical derivations and empirical experiments verify that SNNs converted from pre-trained ANNs achieve comparable classification accuracy to their source ANNs when the inference time step $T \ge 64$ \cite{kim2020strong}. Nevertheless, enlarging the number of time steps linearly multiplies inference latency, which severely hinders the practical deployment of SNNs on neuromorphic chips. For this reason, most state-of-the-art ANN-SNN conversion studies aim to mitigate accuracy degradation under small time steps $T$\cite{hao2025faster,liu2025ecl,xia2026enhancing}, and most of these methods exclusively target quantization errors induced by single-spike firing during ANN-SNN conversion, attempting to eliminate uneven error or quantization error between ANNs and converted SNNs.

However, in practice, the input features follow continuous distributions. Therefore, when the source ANNs adopt continuous activations, it is impossible to eliminate neuron-wise errors after conversion under the spike-based structure of SNNs \cite{hao2023reducing}, and the resulting conversion errors also form smooth distributions \cite{bu2025inference,xu2026error}. When quantization-aware activations are explicitly applied to ANNs \cite{datta2025snn,hu2023fastsnn}, relevant studies generally neglect layerwise error propagation and assume identical inputs fed into successive layers. Under such assumptions, though the activations of converted SNN neurons can be theoretically aligned to eliminate quantization errors, cumulative inter-layer errors inevitably shift the input distribution during forward propagation, leading to mismatches between target ANN activations and realized spike firing rates. This critical issue has rarely been systematically investigated in existing literature. 


Taking prevailing high-fidelity ANN-SNN conversion pipelines as our research subject, we analyze the primary sources of the aforementioned two types of conversion errors and uniformly attribute them to the distribution of residual membrane potential. We further derive the optimal distribution for minimizing conversion errors under small time steps. On this basis, we propose the SRMP method, which introduces a regularization term $\mathcal{L}_{\mathrm{RMPD}}$ to regularize residual membrane potential. Meanwhile, we observe that the performance of the SRMP method is affected by the variance of input distributions. To mitigate this issue, a competitive refinement layer is incorporated to further boost the optimization performance of the SRMP method.

Our core contributions are summarized as follows:
\begin{enumerate}
    \item We conduct an in-depth statistical analysis of residual membrane potential distributions across distinct neuron subgroups
    and elaborate the statistical properties of residual membrane potential under minimal conversion error.
    \item We introduce the regularization term $\mathcal{L}_{\mathrm{RMPD}}$ to modulate the initial membrane potential of IF neurons, and apply fine-tuning with the proposed loss to shift residual membrane potential toward normal distributions. This strategy is compatible with diverse spiking neuron architectures and achieves competitive accuracy under ultra-short time steps $T=2,4,8$.
    \item We propose the SCR-Conv2d competitive refinement layer, which cooperates with SRMP to boost conversion accuracy.
\end{enumerate}

\section{Related work}

\subsection{ANN-SNN Conversion}
Mainstream solutions to optimize ANN-SNN conversion fall into four categories:
(1) Post-conversion fine-tuning: researchers fine-tune converted SNNs via temporal loss functions based on Backpropagation Through Time (BPTT) to adapt network temporal dynamics and narrow the accuracy gap relative to original ANNs \cite{xia2026enhancing,rathi2020enabling,zheng2021going,wu2023tandem}. These methods can exceed the accuracy of the original, but complete  training time via backpropagation is required.
(2) Error-aware neuron pruning: researchers analyze and categorize error-prone neurons arising during ANN-SNN transformation, then conduct targeted weight pruning or neuron pruning on such units \cite{datta2022adversarially,liu2025ecl}.
(3) Spiking neuron parameters optimization: prior work \cite{bu2022optimized,shen2024rethinking} proves that initial membrane potential set to $0.5\theta$ minimizes average conversion error. Researchers further adjust parameters such as firing thresholds or firing mechanisms according to temporal properties of SNNs\cite{sengupta2019going,zhang2024low,wang2026kirin}, which mitigate conversion errors via neuron‑level spiking dynamics. (4) Source ANN architecture modification: 
these representative works\cite{ding2021optimal,han2023symmetric,jiang2023unified,wang2023new,yang2025csqcfs,huang2025residual,bai2026qb,cao2026towards,liu2025ecl,hu2023spiking}, including the QCFS  \cite{bu2023optimal},mainly redesign ANN activation functions to mimic spiking neuron dynamics for high-fidelity conversion.
Other representative modifications involve pooling layers and other intermediate network layers \cite{cao2015spiking,diehl2015fast}, as well as introducing specialized neuron designs such as signed neurons \cite{wang2022signed} and negative spikes \cite{li2022efficient}, these approaches are the most straightforward, yet their upper accuracy bound cannot exceed that of ANNs.

The QCFS scheme, proposed in \cite{bu2023optimal}, is an important model for the subsequent analysis. It replaces the standard ReLU activation with a specialized activation function matching the temporal dynamics of spiking neurons, thereby minimizing the error term in Eq.~\eqref{eq:3.9}. Using the quantization clip-floor-shift (QCFS) activation, the output of ANN neurons mimics discrete spike responses, formulated as
\begin{equation}
a^{l} = \hat{h}(z^{l}) = \lambda^{l} \, \mathrm{clip}\!\left( \frac{1}{L} \left\lfloor \frac{z^{l} L}{\lambda^{l}} + \varphi \right\rfloor, 0, 1 \right).
\label{eq:QCFS-ann-al}
\end{equation}

The original work further theoretically proves that the average conversion error reaches its minimum under the setting $\lambda^l=\theta^l$ and $\varphi=0.5$, which corresponds to an initial membrane potential of $0.5\theta^l$. Nevertheless, the unevenness of conversion errors cannot be completely eliminated even with QCFS, resulting in progressive error accumulation across network layers. Let $\Delta_{l}$ denote the activation discrepancy at layer $l$. The error propagation from layer $l$ to layer $l+1$ satisfies
\begin{equation}
\Delta_{l+1} = W^{l+1} \Delta_{l} + W^l \phi^{l-1}(T) - \left\lfloor \frac{W^l \phi^{l-1}(T)}{\theta^l} \right\rfloor \theta^l.
\end{equation}
The ideal premise of identical input distributions between ANN and SNN is hard to guarantee in practice, which introduces systematic shifts caused by the unevenness of layer-wise conversion errors. Consequently, the neuron-wise output alignment strategy realized by QCFS within ANNs still incurs unavoidable residual errors.

\subsection{Constraints on Membrane Potential}
During real-world SNN training, especially direct training optimized via Backpropagation Through Time (BPTT)\cite{neftci2019surrogate,zhou2023spikformer,zhou2024direct,zhu2022efficient}, the membrane potential distribution of hidden neurons often deviates drastically from the theoretically predicted Gaussian distribution \cite{deweese2006non}, which creates a fundamental obstacle to training stability and final model performance. Membrane  potential is also closely related to conversion error in ANN-SNN conversion, distribution of residual potentials directly determines the magnitude of unevenness error, which persists alongside quantization and clipping errors. To alleviate this issue, existing methods such as RMP-Loss \cite{guo2023rmploss}, MPO\cite{sun2026mpo} and MPD-SGR \cite{jiang2026mpd} explicitly impose constraints on membrane potential statistics to drive the distribution toward the ideal Gaussian shape, and thus boost the adversarial robustness of SNNs. Our SRMP method also centers on residual membrane potential distribution, with a stronger focus on its impact on layer-wise accumulated conversion errors.

\subsection{Lateral Inhibition}

Lateral inhibition is a canonical neuromorphic mechanism observed widely in biological sensory circuits \cite{hartline1956inhibition}, \cite{blakemore1972lateral}. As a core competitive circuit, strongly activated neurons suppress the firing of adjacent neurons via inhibitory connection, which enhances feature contrast and shapes sparse population coding. 
Such competitive activity modulation naturally matches the membrane potential dynamics and discrete spike coding paradigm of SNNs, offering a promising way to regularize neuronal activation distribution. 
Numerous lateral inhibition schemes have been proposed for SNN training \cite{cheng2020lisnn,zhao2022spiking,kim2021inhibitory,zheng2025spiliformer}. Among them, LISNN \cite{cheng2020lisnn} serves as a representative work, which defines local inhibitory neighborhoods. Once a neuron within the neighborhood fires a spike, the inhibition module instantly transmits inhibitory signals to its neighboring neurons to suppress their membrane potentials. We design a lateral-inhibition-based SCR-Conv2d layer for feature enhancement, which is not deployed alone but strengthens the tuning performance of the SRMP method through inhibitory modulation.

\section{Preliminaries}

\subsection{Neuron Model}
Consistent with prior literature \cite{cao2015spiking,diehl2015fast,deng2021optimal}, we adopt the simplest and practically applicable IF neuron model\cite{tal1997computing}. Suppose neurons in layer $l$ receive input activations $x^{l-1}(t)$ propagated from layer $l-1$ at time step $t$. The membrane potential update rule is formulated as:
\begin{equation}
m^{l}(t)=v^{l}(t-1) + W^{l}x^{l-1}(t),
\label{eq:mt-vl-wl}
\end{equation}
where $m^{l}(t)$ represents the instantaneous membrane potential of layer-$l$ neurons, $x^{l-1}(t)$ is the output spike signal of the preceding layer $l-1$, and $W^{l}$ stands for the synaptic weight matrix of layer $l$.

Once the accumulated membrane potential $m^{l}(t)$ exceeds the predefined firing threshold $\theta^{l}$, the neuron emits a spike and its membrane potential is simultaneously reset. The spike generation and reset rules are expressed as three coupled equations:
\begin{equation}
s^{l}(t) = \mathcal{H}\big(m^{l}(t) - \theta^{l}\big),
\end{equation}
\begin{equation}
v^{l}(t) = m^{l}(t) - s^{l}(t)\theta^{l},
\label{eq:vt-ml-sl}
\end{equation}
\begin{equation}
x^{l}(t) = s^{l}(t)\theta^{l},
\label{eq:xl-sl}
\end{equation}
where $\mathcal{H}(\cdot)$ denotes the Heaviside step function, which indicates whether a neuron fires a spike at the current time step.

To precisely retain analog activation information during conversion, we adopt the soft (subtractive) reset mechanism for membrane potential update as proposed in \cite{han2020deep}. Instead of clamping the membrane potential to a fixed constant after spiking, we subtract the firing threshold directly from the overshooting membrane potential to preserve residual potential information.

\subsection{Quantization Error in ANN-SNN Conversion}
Combining Eqs. \eqref{eq:mt-vl-wl}, \eqref{eq:vt-ml-sl} and \eqref{eq:xl-sl}, and summing over all discrete time steps from $1$ to $T$, we derive the following relation:
\begin{equation}
\label{eq:vt-v0}
v^l(T) - v^l(0) = W^l \sum_{i=1}^T x^{l-1}(i) - \sum_{i=1}^T s^l(i) \theta^l.
\end{equation}
Define the time-averaged spike output of layer $l-1$ as
\begin{equation}
\phi^{l-1}(T) = \frac{1}{T}\sum_{i=1}^{T} x^{l-1}(i).
\end{equation}
Treating $\phi^l(T)$ as the effective activation of the converted SNN,we subtract the ANN activation equation from $\phi^l(T)$ and obtain the layer-wise conversion error formulation:
\begin{equation}
\phi^{l} - a^l= \mathbf{W}^{l} \phi^{l-1} - h\big(W^l a^{l-1}\big) - \frac{v^{l}(T) - v^{l}(0)}{T},
\label{eq:3.9}
\end{equation}
where $h(\cdot)$ denotes the element-wise activation function, which is commonly chosen as ReLU for deep ANNs.

\section{Analysis of conversion error}
\label{sec:distributions_of_membrane_potential}



We focus on the IF neuron model, where spike activations from presynaptic neurons are modeled as Poisson processes $k_j \sim \text{Poisson}\left(\nu_j \Delta t\right)$ \cite{gerstner2002spiking,burkitt2006review,scarpetta2013effects}. Here, $k_j$ denotes the spike count and $\nu_j$ represents the mean firing rate. The membrane potential increment of postsynaptic neuron $i$ is formulated as
\begin{equation}
    \Delta V_i = \sum_{j=1}^{M} w_{ij} \cdot k_j \cdot \theta ,
\end{equation}
which corresponds to the summation of $M$ independent Poisson random variables. By the Central Limit Theorem (CLT) \cite{feller1971introduction}, the distribution of $\Delta V_i$ asymptotically converges to a normal distribution as the number of input synapses grows:
\begin{equation}
\begin{split}
\Delta V_i \xrightarrow[M\to\infty]{\text{CLT}}
\mathcal{N}\biggl(\sum_{j=1}^{M} w_{ij} \nu_j \Delta t \theta,\,
\sum_{j=1}^{M} w_{ij}^2 \nu_j \Delta t \theta^2\biggr).
\end{split}
\label{eq:clt_dist}
\end{equation}

All the following analyses assume the approximated input distribution to be Gaussian.

\subsection{Distribution of Residual Membrane Potential under Quantization}
\label{subsec:qcfs_rmp_distribution}
For conciseness, we abbreviate residual membrane potential as RMP and residual membrane potential distribution as RMPD in the following text. We conduct a detailed analysis on the membrane potential statistics of samples from the same category within the quantization framework, with particular focus on the distribution of RMP.

We adopt the simulated membrane potential derived from the ANN branch for analysis. According to Eq.~\eqref{eq:QCFS-ann-al}, the RMP of IF neurons at time $T$ under the ANN formulation is formulated as
\begin{equation}
u^l = a^l - \lambda^{l} \, \mathrm{clip}\!\left( \frac{1}{L} \left\lfloor \frac{z^{l} L}{\lambda^{l}} + \varphi \right\rfloor, 0, 1 \right).
\end{equation}

Consider all input samples belonging to a single class. Classical statistical theory states that all intra-class samples follow an identical data-generating distribution. Based on the Central Limit Theorem (CLT), the input signal received by each neuron approximately obeys a normal distribution.
Formally, $a^l \sim \mathcal{N}(\mu_l, \sigma_l^2)$, where $\mu_l$ denotes the input mean and $\sigma_l^2$ denotes the input variance.

Without loss of generality, define $\lceil \frac{\max(a^l)L}{\theta } + 0.5 \rceil = M_i$, where $\lceil \cdot \rceil$ represents the ceiling operator. Combining this definition with Eq.~\eqref{eq:QCFS-ann-al}, when $a^l\in[\frac{\lambda^l \cdot k_i}{L},\frac{\lambda^l \cdot (k_i+1)}{L}]$, we have
$u^l = a^l - \frac{\lambda^l \cdot k_i}{L}$, with $k_i = \lceil \frac{a^l L}{\theta} + 0.5 \rceil$ and $k_i \in \{0, 1, \dots, M_i\}$. Under this condition, the RMP denoted by $u^l$ also follows a normal distribution, i.e., $u^l \sim \mathcal{N}\big(\mu_l-\frac{\lambda^l \cdot k_i}{L}, \sigma_l^2\big)$.

From the above derivation, the full support of input $a^l$ is partitioned into $M_i$ equal-length subintervals, each with length $\frac{\theta}{L}$. 
As a result, the overall RMP distribution is essentially a mixture of $M_i$ truncated normal distributions with identical variance, where each component is truncated onto $[0, \frac{\theta}{L}]$.

If one single subinterval dominates the mixture, the aggregated overall distribution can be approximated by a single Gaussian. 
However, for most neurons, the RMP distribution is jointly shaped by multiple input subintervals, which leads to visible deviation from the normal distribution and usually forms a multimodal profile, where each mode corresponds to the mean of the normal distribution of one subinterval. A concrete illustration is provided in Fig.~\ref{fig:four-grid}.

\begin{figure}[!t]
  \raggedbottom
  \centering
  \begin{minipage}{0.9\columnwidth}
    \centering
    \includegraphics[width=\linewidth]{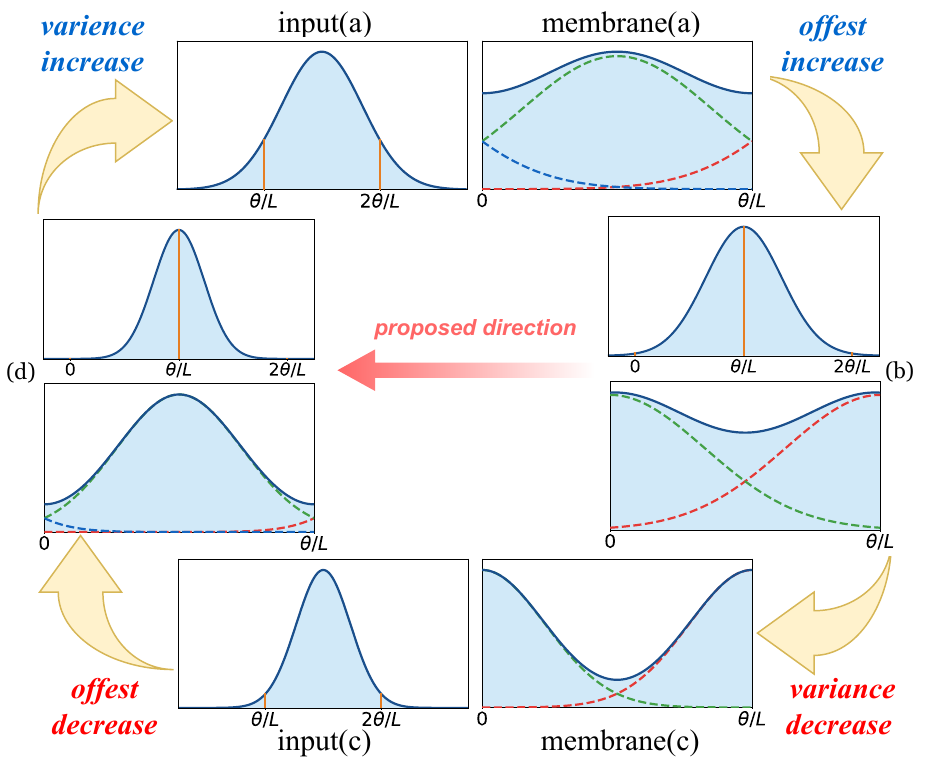}
  \end{minipage}
  \caption{Distribution of residual membrane potential. The subplots demonstrate RMPD under distinct input distributions. Case (a) is dominated by the interval $[\frac{\theta}{L},\frac{2\theta}{L}]$, while case (b) receives equal contributions from $[0,\frac{\theta}{L}]$ and $[\frac{\theta}{L},\frac{2\theta}{L}]$. Correspondingly, the RMP distribution in (a) exhibits near-Gaussian behavior, whereas the distribution in (b) severely departs from Gaussianity. Figure (c),(d) correspond to figure (a),(b) with smaller variance. Comparing case (c),(d) with (a),(b), we can tell that reducing  variance of neuron input amplify the error-reduction effect. }
  \label{fig:four-grid}
\end{figure}

In what follows, we discuss the benefits of Gaussian-like RMP distributions over multimodal distributions, which yields Theorem~\ref{thm:theorem_1}.
Compared with unimodal distributions\cite{jiang2025adaptive,yan2025training}, multimodal RMP distributions yield higher probability density near the interval boundaries, i.e., the neighborhoods of $0$ and $\frac{\theta}{L}$ (as proven in Theorem 1).
Such boundary-concentrated RMP serves as the primary source of mismatch between the spike firing rates of SNNs and the activation outputs of source ANNs.
Specifically, consider the $i$-th neuron in layer $l$. Without the ideal assumption that the mean outputs of the preceding layer are identical for the ANN and SNN ($a^{l-1} = x^{l-1}$), the layer-wise propagated mean activation $a^{l}$ of the ANN cannot offset the steady-state mean spike output of the SNN, where $x^l=\phi^l$ denotes the time-averaged output of layer $l$.The membrane potential update rule for SNN neurons follows
\begin{equation}
    u^{l}_{i}(t+1) = \mathrm{clip}\bigl(u^{l}_{i}(t) + x^{l}_{i} - \theta \, z^{l}_{i}(t),\, 0,\, \theta\bigr) .
\end{equation}
Without loss of generality, we define the inherent input bias as
$
\delta_{\mathrm{in},i} = \frac{x^{l}_{i}}{L} - a^{l}_{i} > 0.
$ 
When the normalized residual membrane potential $u^l$ of the ANN falls within the critical interval
$\Bigl(\frac{\theta}{L} - \delta_{\mathrm{in},i},\ \frac{\theta}{L}\Bigr),$
the quantization clipping constraint yields the following relation:
\begin{equation}
\begin{split}
\frac{x^{l}_{i}}{L}
&= \delta_{\mathrm{in},i} + a^{l}_{i}
= \delta_{\mathrm{in},i} + u^l + \theta \, \clip\!\left(\frac{1}{L}\biggl\lfloor \frac{z^{l} L}{\theta} + \varphi \biggr\rfloor\right) \\
&> \frac{\theta}{L} + \theta \cdot \clip\left( \frac{1}{L} \left\lfloor \frac{z_i L}{\theta} + \varphi \right\rfloor \right).
\end{split}
\label{eq:quant_clip}
\end{equation}
Accordingly, under sufficient spike firing in the SNN, the actual spike count of neuron $i$ is one more than the approximate spike activation derived from the ANN. Conversely, when the residual membrane potential $u^l$ takes a small value,the SNN spike count becomes one less than the theoretical value predicted by the ANN.

In summary, the boundary-heavy multimodal RMP distribution induced by quantization introduces systematic clipping error, which ultimately leads to an inherent mismatch between the mean activation of the ANN and the statistical mean spike output of the converted SNN.

\subsection{{Theoretical Analysis}}
The following theorem elaborates the relationship between the membrane potential distribution and the set of intervals $U$,and identifies the optimal RMPD that minimizes the conversion error under certain conditions. For the sake of concise proof derivation, the initial membrane potential $m_{\mathrm{init}}$ is temporarily omitted in our analysis.

\begin{theorem}
\label{thm:theorem_1}
(1) For an interval of fixed length defined over a normal distribution $\Phi$, the integral over this interval attains its maximum value if and only if the mean of the normal distribution coincides with the midpoint of the interval.

(2) Let the set of intervals be 
\begin{equation}
U = \biggl\{\biggl[0, \frac{\theta}{L}\biggr], \biggl[\frac{\theta}{L}, \frac{2\theta}{L}\biggr], \dots, \biggl[\frac{(L-1)\theta}{L}, \theta\biggr]\biggr\}. 
\end{equation}
Denote $U_T$ as the special case where exactly one interval satisfies the extremum condition stated in (1). 
The corresponding RMP distribution $F_{U_T}(x)$ is symmetric, and its gradient equals zero at both the midpoint and boundaries of each subinterval in $U_T$.

(3) When $L=4$ and $\frac{\theta}{L} > 2\sigma$, the boundary probability density of the distribution defined in (2) reaches its minimum. Formally, for any given $\epsilon>0$, define the boundary region $B=[0,\epsilon] \cup [\frac{\theta}{L}-\epsilon, \frac{\theta}{L}]$. The following relation holds:
\begin{equation}
\inf_{U}\int_{B} f_U(x) dx = \int_{B} f_{U_T}(x) dx,
\end{equation}
where $f_U(x)$ stands for the probability density function formed by superimposing normal distributions truncated on each subinterval of $U$. This indicates that $U_T$ yields the locally minimal probability density near interval boundaries.

In statements (1) and (2), we adopt the ideal assumption that the distribution strictly follows a complete normal distribution $\Phi=\mathcal{N}(\mu, \sigma^2)$. To better fit real input characteristics in statement (3), we approximate the distribution with symmetric truncated normal distributions, where the normalization constant $c$ satisfies $c \to 1$.
\end{theorem}

Proof of the Theorem \ref{thm:theorem_1} is deferred to Appendix. Experimental validation demonstrates that this condition holds for the vast majority of neurons, which is detailed in Section~\ref{subsection:detail}.
Furthermore, the conclusion derived under $L=4$ can be generalized to arbitrary time-step settings, which verifies the universal validity of statement (3).

The theorem above reveals that shifting the interval covering the input mean such that the input mean coincides exactly with the interval midpoint yields a symmetric unimodal RMPD close to Gaussian. Meanwhile, this adjusted distribution minimizes conversion errors from the perspective of residual membrane potential statistics.

\section{Methods}
\subsection{{SRMP-based Regularization}}
Based on the analysis in the previous section, we need to adjust the overall membrane potential distribution to satisfy the condition given in Theorem \ref{thm:theorem_1}. To this end, we propose a regularization strategy named SRMP, which reduces conversion errors by shifting the membrane potential distribution.

\subsubsection{\textbf{Initial Membrane Potential Adjustment}}
\label{subsec:init_membrane_adjust_reg}
To realize the shifting of input membrane potentials, we modify the fixed initial membrane potential $\frac{\theta}{2}$ of IF layers to $m_{\mathrm{init}} \cdot \theta$, where $m_{\mathrm{init}}$ denotes a trainable parameter initialized to 0.5. Taking the adjustable initial membrane potential into account, the input range is partitioned into the following set of intervals:
\begin{equation}
\begin{split}
U=\biggl\{
&\bigl[-m_{\mathrm{init}}\theta,\ \bigl(\tfrac{1}{L}-m_{\mathrm{init}}\bigr)\theta\bigr],\quad\\
&\bigl[\bigl(\tfrac{1}{L}-m_{\mathrm{init}}\bigr)\theta, \bigl(\tfrac{2}{L}-m_{\mathrm{init}}\bigr)\theta\bigr], \\
&\quad\dots,\quad\\
&\bigl[\bigl(\tfrac{L-1}{L}-m_{\mathrm{init}}\bigr)\theta,\ \bigl(1-m_{\mathrm{init}}\bigr)\theta\bigr]
\biggr\}.
\end{split}
\end{equation}

This partition corresponds to a global interval shift of $-m_{\mathrm{init}}\theta$. According to Theorem 1, for samples belonging to class $n$, let $\mathbb{E}[I_n^i]$ denote the mean input signal received by the $i$-th neuron. We define the loss term $L_n^i$ as the squared distance between the midpoint of the target interval and $\mathbb{E}[I_n^i]$.

\begin{equation}
D(x) = x - \mathrm{round}(x),
\end{equation}
\begin{equation}
d_n^i = \frac{\mathbb{E}(I_{n}^i)}{\theta} L + m_{\mathrm{init}} - 0.5,
\end{equation}
\begin{equation}
L_n^i = \lambda_{n} \left\| D(d_{n}^i) \right\|_{2}^{2}.
\end{equation}
Here, $d_n^i$ represents the shift offset, $\lambda_n$ is the regularization weight, and $\mathrm{round}(x)$ stands for the rounding function that maps $x$ to its nearest integer. Class-specific regularization weights $\lambda_n$ are introduced to control the update magnitude and bias direction of $m_{\mathrm{init}}$, enabling customized adjustment for each individual class.

The average regularization loss of the $i$-th neuron across all $N$ classes is formulated as
\begin{equation}
\mathcal{L}_{\mathrm{RMPD}}^i=\frac{1}{N}\sum_{n=1}^N \lambda_n \cdot L_{n}^i.
\end{equation}
By averaging over all neurons within the IF layer, we derive the layer-wise regularization loss:
\begin{equation}
\mathcal{L}_{\mathrm{RMPD}} = \frac{1}{M}\sum_{i=1}^M \mathcal{L}_{\mathrm{RMPD}}^i.
\end{equation}
Here, $M$ represents the total number of neurons in the corresponding IF layer. We incorporate $\mathcal{L}_{\mathrm{RMPD}}$ of every IF layer into the overall training objective. The total loss function is written as
\begin{equation}
\label{eq:total-loss}
\mathcal{L}_{\mathrm{total}}=\mathrm{Loss}+\boldsymbol{\lambda}\cdot \mathcal{L}_{\mathrm{RMPD}},
\end{equation}
where $\boldsymbol{\lambda}$ denotes the hyperparameter controlling the strength of regularization. The fine-tuning procedure for the initial membrane potential $m_{\mathrm{init}}$ is visualized in Fig.~\ref{fig:gauss-spikes}, and the complete algorithm pseudocode is summarized in Algorithm~\ref{alg:dynamic_mem_ANN}.

\begin{figure}[!t]
    \raggedbottom
    \centering    \includegraphics[width=0.95\columnwidth]{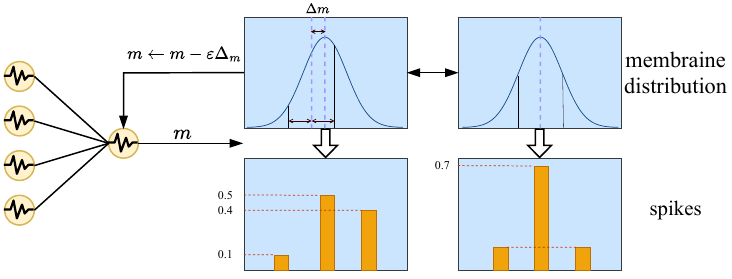}
    \caption{Fine-tuning of the initial membrane potential $m_{\mathrm{init}}$. The gradient step for updating $m_{\mathrm{init}}$ is proportional to the offset of the target interval. After each adjustment, the resulting RMPD becomes closer to a normal distribution, and the corresponding spike activations cluster more tightly around the characteristic firing frequency of each neuron.}
    \label{fig:gauss-spikes}
\end{figure}

\begin{algorithm}[t]
  \caption{Training ANN with dynamic $m_{\mathrm{init}}$ and membrane distribution loss}
  \label{alg:dynamic_mem_ANN}
  \begin{algorithmic}[1]
    \REQUIRE 
      Pretrained ANN model $\mathcal{M}_{\mathrm{ANN}}$ with initial $m_{\mathrm{init}}$;
      Total training iterations per epoch $I_{\mathrm{train}}$;
      Regularization weight $\lambda$;
      Learning rate $\varepsilon$;
      Training dataset $\mathcal{D}$;
      $\mathcal{L}_{\mathrm{RMPD}}$ function $\mathcal{L}(\boldsymbol{\mu},m_{\mathrm{init}})$, where $\boldsymbol{\mu}$ is the input mean tensor.
    \ENSURE ANN model with dynamic $m_{\mathrm{init}}$

    \FOR{$i = 1$ \TO $N_{\mathrm{epoch}}$}
      \FOR{each sample in $\mathcal{D}$}
        \STATE Sample miniBatch $(x, y)$ from $\mathcal{D}$
        \STATE $X_{\mathrm{output}} = \mathcal{M}_{\mathrm{ANN}}(x)$
        \STATE $\mathrm{Loss} = \mathrm{CrossEntropy}(X_{\mathrm{output}}, y)$
        \STATE $\boldsymbol{\mu} = \text{Tensor obtained by averaging inputs of each }$
        \STATE \quad$\text{class in the miniBatch}$
        \STATE $\mathcal{L}_{\mathrm{RMPD}} = \mathcal{L}(\boldsymbol{\mu},m_{\mathrm{init}})$
        \FOR{$l = 1$ \TO $\mathcal{M}_{\mathrm{ANN}}.\mathrm{layers}$}
          \STATE $m_{\mathrm{init}}^l \leftarrow m_{\mathrm{init}}^l - \varepsilon\left(
          \frac{\partial \mathrm{Loss}}{\partial m_{\mathrm{init}}^l}
          + \lambda \frac{\partial \mathcal{L}_{\mathrm{RMPD}}}{\partial m_{\mathrm{init}}^l}
          \right)$
        \ENDFOR
      \ENDFOR
    \ENDFOR
  \end{algorithmic}
\end{algorithm}

\subsubsection{\textbf{Dynamic Regularization Weight}}
\label{subsec:dynamic_reg_weight}

Intuitively, the magnitude of $\lambda_n$ is determined by two factors:

1. Classes with larger interval offsets should be prioritized. The regularization weight should be positively correlated with the loss value to accelerate the convergence of $m_{\mathrm{init}}$.

2. Input variances differ across distinct classes. For classes with small input variance, the input and output distributions are compact, yielding an RMPD that is inherently close to Gaussian with minor conversion error. Therefore, $\lambda_n$ should increase proportionally with the input variance.

To derive a simple yet effective weighting strategy, we establish the following theorem.

\begin{theorem}
\label{thm:theorem_2}
For inputs following an approximate normal distribution, the cumulative probability over the target interval $[m-\frac{\theta}{2L},m+\frac{\theta}{2L}]$ is negatively correlated with the input variance and negatively correlated with the interval offset $|\mu-m|$. Here, $\theta$ denotes the membrane firing threshold, $\mu$ is the mean of the normal input distribution, and $m$ represents the midpoint of the target interval.
\end{theorem}

Proof of the Theorem \ref{thm:theorem_2} is deferred to Appendix. From Theorem \ref{thm:theorem_2}, we parameterize the regularization weight as $\lambda_n = f(p)$ with the derivative satisfying $f' < 0$, where $p$ denotes the cumulative probability over the target interval. To amplify the weight discrepancy across different classes, concentrate the adjustment bias of $m_{\mathrm{init}}$ on the classes that require the most correction, we enforce the limit condition $\lim_{p \to 1} f(p) = 0$. Two candidate smooth functions are adopted for $f(p)$: the sigmoid-based form 
\begin{equation}
f(p)=1-\mathrm{sigmoid}(k \cdot p)=\frac{1}{1+\exp(kp)}, 
\end{equation}
and a steeper variant 
\begin{equation}
f(p)=\exp\big[k(p-0.5)\big]^{2l},
\end{equation}
where $k$ and $l$ are fixed hyperparameters. Examples of RMPD before and after adjustment of some IF neurons are shown in Fig. \ref{fig:pre-post-normalization}.

\begin{figure}[!t]
  \raggedbottom
  \centering
    \includegraphics[width=\linewidth]
    {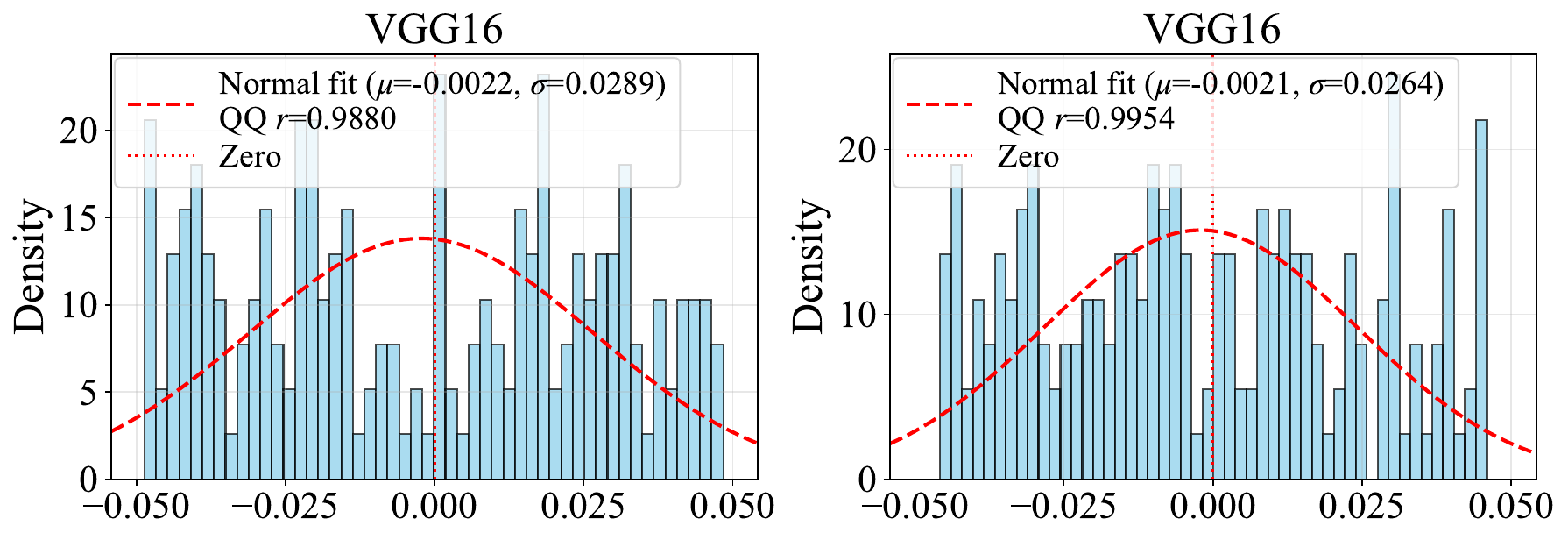}
    \includegraphics[width=\linewidth]
    {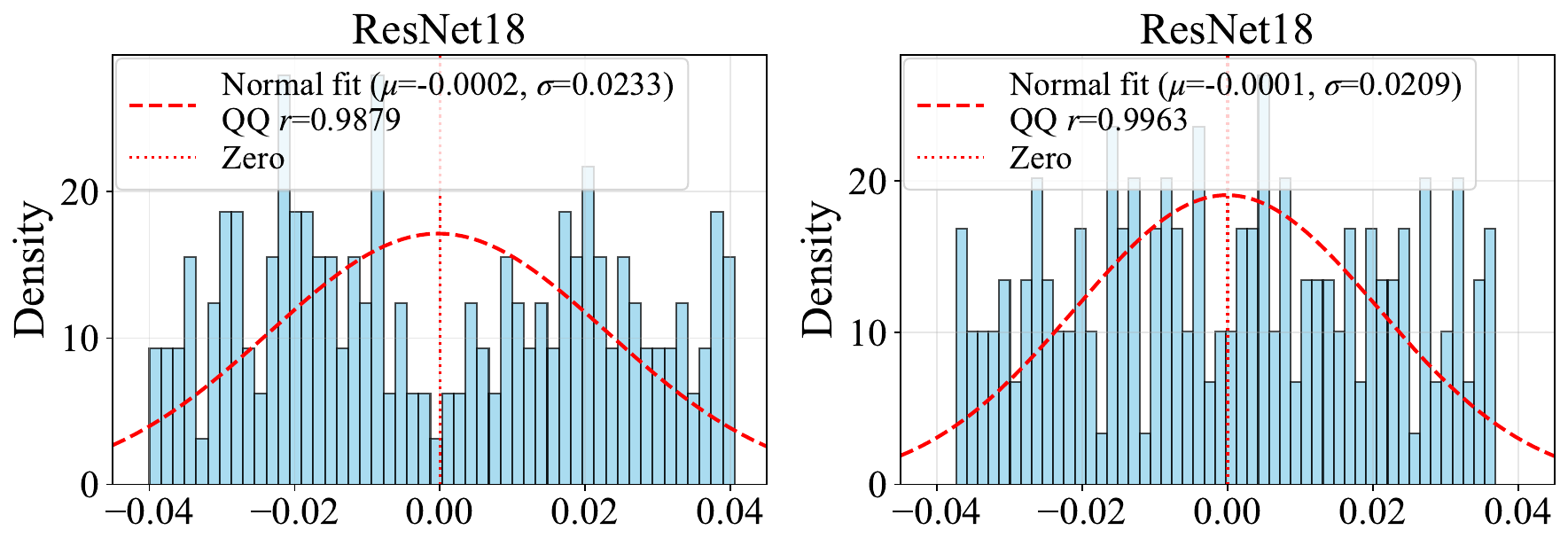}
    \includegraphics[width=\linewidth]
    {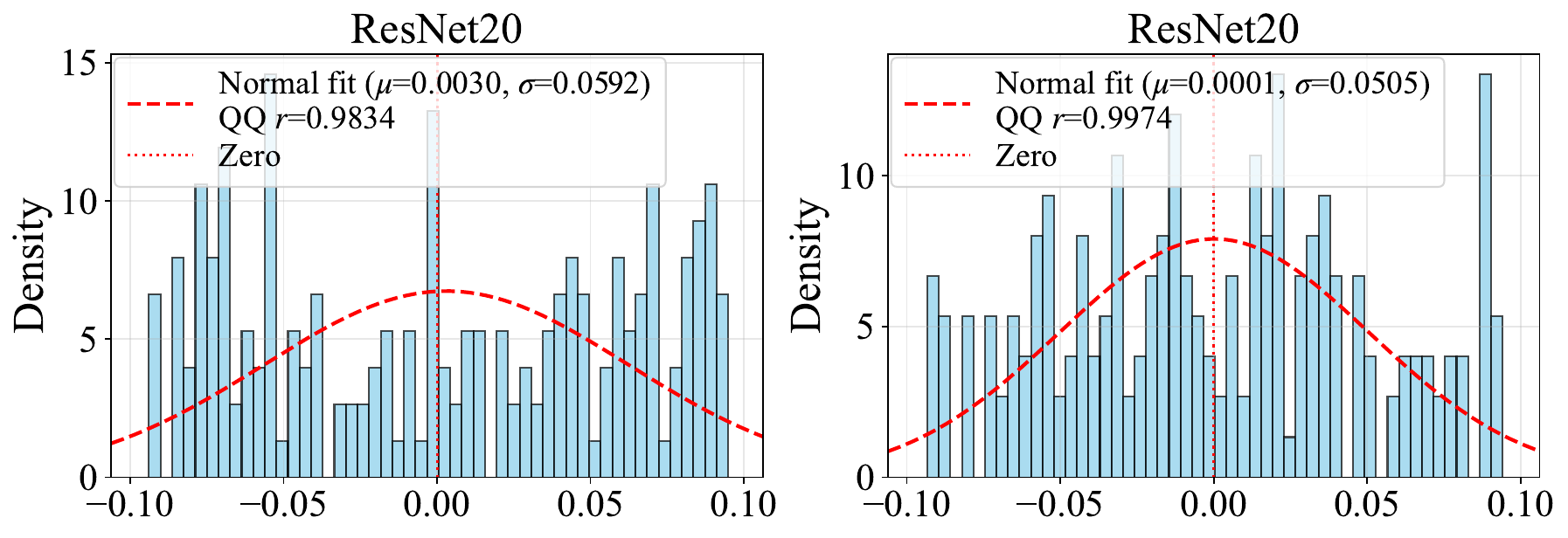}    
    \hspace{2em}
    {\footnotesize (a) pre-normalization }
    \hspace{5em}
    {\footnotesize (b) post-normalization }
  \caption{RMPD shapes before and after adjustment. Parameter QQ quantifies the similarity to Gaussian distribution, and it can be observed that the regularized RMPD is closer to Gaussian distribution after adjustment.}
  \label{fig:pre-post-normalization}
\end{figure}

\subsubsection{\textbf{Regularizer Design for Networks with General Activations}}
\label{subsubsec:general_Activations}
Quantized activation function in QCFS model exhibits natural compatibility with the SRMP regularization.
General neural networks without nonlinear activations share similar conversion error characteristics with the Quantized framework, and if the input distribution approximately follows a Gaussian distribution, the proposed $\mathcal{L}_{\mathrm{RMPD}}$ regularizer is still applicable with minor modifications, which we elaborate as follows.

Unlike analysis in Section \ref{sec:distributions_of_membrane_potential}, which we mitigates conversion errors by regulating the RMPD distribution, Eq.~\eqref{eq:vt-v0} indicates that the conversion error of activation-free networks directly equals $\text{RMP} - m_{\mathrm{init}}$. When the default initial membrane potential is fixed as $m_{\mathrm{init}}=\frac{\theta}{2}$, we aim to minimize overall error while preserving the mean value of membrane potentials. This requires concentrating the RMPD around $\frac{\theta}{2}$. 
Based on the analysis in Section~\ref{subsec:qcfs_rmp_distribution} and Theorem 1, the regularizer $\mathcal{L}_{\mathrm{RMPD}}$ remains valid for tuning $m\mathcal{L}_{\mathrm{init}}$. However, when the RMPD converges toward the midpoint $\frac{\theta}{2}$ of the interval $[0,\theta]$, any deviation of $m_{\mathrm{init}}$ from $0.5$ shifts the overall distribution mean and consequently amplifies conversion errors. To address this issue, we introduce a constant shift compensation for neuron outputs:
\begin{equation}
x_c^l=x^l+\big(0.5-m_{\mathrm{init}}\big)\theta,
\end{equation}
as illustrated in Fig.~\ref{fig:RMPD-none-Relu}. This compensation term remains constant during inference and training, introducing no additional computational overhead.

For the regularization weight $\lambda_n$, we adopt the formulation $\lambda_n=f(d^k\cdot p)$ with $f'>0$. A proper $k$ is selected such that $\lambda_n$ is positively correlated with the interval offset $d$. The theoretical justification for this weight design is provided in the Appendix.


\begin{figure}[!t]  
    \raggedbottom
    \centering      \includegraphics[width=0.95\columnwidth]{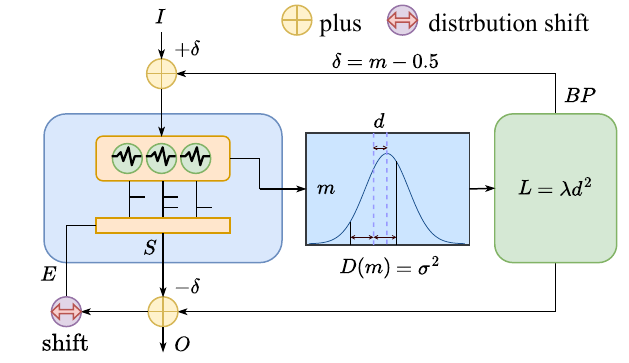}
    \caption{Fine-tuning of initial membrane potential with shift compensation}  
    \label{fig:RMPD-none-Relu}  
\end{figure}

\subsubsection{\textbf{Extension to Networks with Dual-Polarity Thresholds and Multi-Thresholds}}
For spiking models that support negative spikes, minor modifications can be applied to adapt our method. The spike generation rules are formulated as $s_-=\mathcal{H}(m-m_{\mathrm{init}}\cdot\theta_-)$ and $s_+=\mathcal{H}(-m-\theta_++m_{\mathrm{init}}\cdot\theta_+)$, where $\theta_+$ and $\theta_-$ denote the positive and negative firing thresholds, respectively.

For positive spikes, the equivalent initial membrane potential is $(m_{\mathrm{init}}-1)\cdot\theta$, and each positive spike contributes $+1$ to the total spike count. For negative spikes, the initial membrane potential is set to $m_{\mathrm{init}}\cdot\theta$, with each negative spike counted as $-1$. Consequently, the total spike count equals zero within the interval $[(m_{\mathrm{init}}-1)\theta,m_{\mathrm{init}}\theta]$ of length $\theta$, which still satisfies the equal-length subinterval assumption adopted in our derivation.

To validate the effectiveness of the proposed strategy, we conduct fine-tuning experiments on Transformer architectures based on the ECMT conversion model equipped with multi-threshold neurons (MT-N) \cite{huang2024highperformance}. 
Since the native Transformer spiking backbone does not incorporate a trainable $m_{\mathrm{init}}$ term, we simplify the optimization procedure by discarding the standard cross-entropy loss and solely employing $\mathcal{L}_{\mathrm{RMPD}}$ as the training objective: $\mathcal{L}_{\mathrm{total}}=\mathcal{L}_{\mathrm{RMPD}}$.

\subsection{Spike Competitive Refinement}

\label{subsec:lateral_inhibition_snn_training}

Based on the membrane potential statistical analysis in Section~\ref{sec:distributions_of_membrane_potential}, hidden-layer neurons fail to form highly specialized response patterns when the network’s inherent feature extraction and discrimination capacity is limited. For samples belonging to the same class, different neurons produce blurry and redundant activations responding to similar or correlated features, which manifests as spatial and temporal dispersion of spike activations. Such dispersion cannot be completely eliminated even if the initial membrane potential is tuned to its optimal value. To resolve this fundamental limitation from within the network architecture, we introduce the Spike Competitive Refinement (SCR) mechanism. 
SCR enhances feature contrast, which equivalently reduces the variance of neuron input signals and stabilizes input distributions. This tighter input concentration further compresses the spread of RMPD, amplifying the error-reduction effect brought by optimized initial membrane potentials, which is shown in Fig. \ref{fig:four-grid}. Meanwhile, more neurons fall into identical quantized subintervals, which improves the stability of spike firing patterns.


To clarify the propagation pipeline of spike generation and competition, we insert the SCR-Conv2d layer after the IF layer following standard convolution layers, as illustrated in Fig.~\ref{fig:LI_Conv2d}. The layer placed after IF layers nonlinearly attenuates IF outputs to a proper extent while retaining the computational efficiency brought by spike-based representation. A detailed complexity analysis is provided in Section \ref{sec:complexity_analysis}.

\begin{figure*}[t]
    \raggedbottom
    \centering    \includegraphics[width=0.98\textwidth]{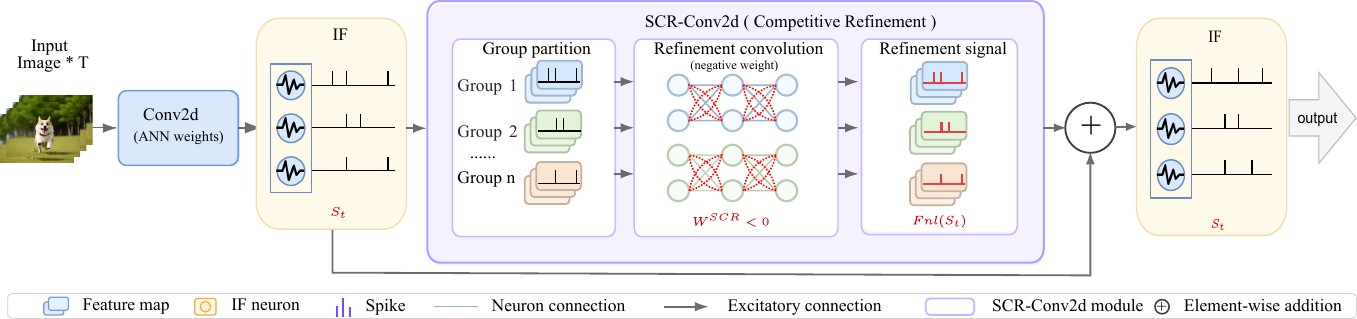}
    \caption{Architecture of SCR-Conv2d. The SCR-Conv2d layer is inserted after the standard convolutional layer. The outputs of the IF layer are transformed into refinement signal magnitudes via the nonlinear function $F_{\mathrm{nl}}$. An additional IF layer is appended after SCR-Conv2d to maintain spike transmission. This figure illustrates the spike-driven information propagation pipeline of the converted SNN within VGG-16.}
    \label{fig:LI_Conv2d}
\end{figure*}

We implement feature enhancement via a grouped convolution layer named SCR-Conv2d, where each channel corresponds to an independent group. Neurons within the same channel mutually suppress one another through trainable convolution kernels. Refinement signals directly act on spike sequences since membrane potentials do not exist within ANNs. In the ANN forward pass, the output signal after SCR-Conv2d  is formulated as:
\begin{equation}
a_{SCR}=a_{IF}+\mathrm{conv}(act,W,C).
\end{equation}
Here, $a_{IF}$ denotes the output of the IF layer; $act=F_{\mathrm{nl}}(f_{IF})$ is the nonlinear activation; $W$ stands for convolution kernel weights; $f_{IF}=\frac{a_{IF}}{L}$ represents the spike firing rate equivalent in the ANN branch; and $C$ denotes the total number of channels. All entries of $f_{IF}$ take values from the discrete set $\{0,\frac{1}{L},\frac{2}{L},\dots,1\}$. Accordingly, we define $F_{\mathrm{nl}}(\frac{i}{L})=s_i$, where each $s_i$ is a trainable parameter. To guarantee spike-form inputs for subsequent convolution layers, an additional IF layer is appended behind every SCR-Conv2d module.

{Conversion of SCR-Conv2d:}
When converting the ANN model into an equivalent SNN, three signal transformation schemes are proposed to maintain consistent behavior before and after conversion. Let $S_A$ denote the refinement signal derived from the ANN branch. 
\begin{enumerate}
    \item \textbf{First time-step mode}: refinement is only applied at the first time step with total refinement $S=S_A T$.
    \item \textbf{Every time-step mode}: The refinement signal is uniformly distributed across all time steps, yielding per-step strength $S=S_A$.
    \item \textbf{Spike-driven mode}: The instantaneous refinement signal is $S=\frac{S_A}{f} \cdot s$, where $s=\mathbb{I}\left(a_{\mathrm{IF}} >0\right)$ denotes the indicator function.
\end{enumerate}

We adopt spike-driven mode as the default setting in all experiments, as this scheme best matches the spike-triggered refinement mechanism observed in biological neurons.

\section{Experiment}
\label{chap:chapter7}

\subsection{{Experiment settings}}

\subsubsection{\textbf{Network Architectures}}
In this section, we validate the performance of two proposed strategies: fine-tuning the initial membrane potential via the designed regularizer, and integrating the SCR-Conv2d module. We name these two methods SRMP and SCR-SRMP, respectively.
Both schemes operate on pre-trained or converted SNN models, where regularization constraints are applied to fine-tune the initial membrane potentials of IF layers placed after fully-connected (FC) layers and convolutional layers. Compared with SRMP, SCR-SRMP inserts trainable SCR-Conv2d layers into the network backbone. If SCR-Conv2d layers are involved in the proposed methods, we first train the network equipped with SCR modules, and then perform fine-tuning on the initial membrane potentials. After completing full model pre-training, we conduct the secondary fine-tuning stage for a 10 epoches to adjust the initial membrane potentials of IF layers. To stabilize model behavior, firing thresholds and network weights are frozen during this stage, and we retain the checkpoint with the optimal validation performance. Note that the dynamic weight $\lambda_n$ depends on input statistics. To stabilize regularization loss, we update $\lambda_n$ once every 10 training iterations.

\subsubsection{\textbf{ANN-to-SNN Settings}}
For ANN-to-SNN conversion, we compare our approaches against state-of-the-art image classification methods on three standard benchmarks: CIFAR-10 \cite{krizhevsky2009learning}, CIFAR-100 \cite{krizhevsky2009learning}, and ImageNet \cite{deng2009imagenet} under QCFS framework. Following prior literature, we adopt VGG-16 \cite{simonyan2014very}, ResNet-18, ResNet-20 and ResNet-34 as source ANNs. For SCR-Conv2d layer initialization, we adopt a kernel size of 3. All kernel entries except the diagonal elements are initialized to negative values to avoid self-refinement of neurons. During training, a clamp function is applied to enforce non-positive kernel weights. For training simplicity, we replace the generic nonlinear mapping inside SCR-Conv2d with the sigmoid function. During fine-tuning, the batch size is set to 500, and the learning rate is 0.1. Regularization weight varies with the number of classes in the dataset, which is typically set to 4000 for CIFAR-10, 200 for CIFAR-100, and 1 for ImageNet.

\subsubsection{\textbf{Transformer-to-SNN Settings}}
Additionally, we evaluate the effectiveness of the SRMP regularizer on ECMT architectures\cite{huang2024highperformance} with multi-threshold neurons. Details of finetuning are similar to ANN-SNN conversion,  except that batch size is set to 64. This leads to a critical issue that, for datasets with a large number of classes, each class appears at most once in an individual batch. Such setting introduces excessive randomness and lacks statistical validity. To address this problem, we devise adjustments for batch sampling and the statistics of fine-tuning measurements, as detailed in \ref{subsection:detail}.

\subsection{Comparison with Other Competitive CNN-based Conversion Algorithms}
We compare our method against several mainstream high-performance ANN-SNN conversion algorithms on CIFAR-10, CIFAR-100 and ImageNet datasets, including QCFS\cite{bu2023optimal}, SRP from Hao et al.\cite{hao2023reducing}, SNM from Wang et al.\cite{wang2022signed}.

Full accuracy comparisons are listed in Table~\ref{tab:cifar10_100_result_2}. On CIFAR-10, CIFAR-100 datasets, our model outperforms almost all the other conversion methods when T=2,4,8 (note that SRP needs 4 extra time steps). Besides, on ImageNet, SCR-SRMP achieves an accuracy of $56.46\%$ at 16 timesteps for VGG16, $5.54\%$ higher compared with QCFS, and also outperforms QCFS at other time steps. At larger time steps, the baseline accuracy saturates with little potential for further improvement. In addition, the assumptions of the theorem tend to break down, as reflected in the nearly‑uniform RMPD , which leads to marginal performance improvement.

\begin{table*}[!t]
  \raggedbottom
  \centering
  \setlength{\tabcolsep}{10.pt}  
  \caption{Comparison with state-of-the-art methods on  the CIFAR-10,CIFAR-100 and ImageNet dataset}
  \label{tab:cifar10_100_result_2}
\begin{tabular}{l l l c c c c c c c}
\toprule
Method & Net. & Param. (M) & ANN & 2 & 4 & 8 & 16 & 32 & 64 \\
\midrule
\multicolumn{9}{c}{\textbf{CIFAR-10}} \\
\midrule
QCFS\cite{bu2023optimal} & VGG16 & 138.36 & 95.78 & 90.67 & 94.21 & 95.40 & 95.73 & 95.74 & 95.83 \\
SNM\cite{wang2022signed} & VGG16 & 138.36 & 95.72 & 90.55 & 94.54 & 95.56 & 95.78 & 95.77 & 95.82 \\
SRP($\tau$=4)\cite{hao2023reducing} & VGG16 & 138.36 & 95.62 & $-$ & 95.43 & 95.55 & 95.58 & 95.57 & 95.60 \\
NQQCFS\cite{huang2025residual} & VGG16 & 138.36 & 95.21 & 91.93 & 94.80 & 95.48 & 95.70 & 95.79 & 95.72 \\
\textbf{SCR-SRMP} & VGG16 & 138.36 & 95.65 & \textbf{92.84} & \textbf{94.80} & \textbf{95.66} & \textbf{95.78} & \textbf{95.80} & \textbf{95.86} \\
\midrule
QCFS\cite{bu2023optimal} & ResNet-18 & 11.69 & 96.48 & 92.86 & 94.92 & 96.21 & 96.55 & 96.64 & 96.52 \\
SRP($\tau$=4)\cite{hao2023reducing} & ResNet-18 & 11.69 & 96.58 & $-$ & 95.25 & 95.60 & 95.55 & 95.55 & 95.58 \\
NQQCFS\cite{huang2025residual} & ResNet-18 & 11.69 & 95.52 & 93.72 & 95.37 & 96.21 & 96.38 & 96.31 & 96.36 \\
ECL\cite{liu2025ecl} & ResNet-18 & 11.69 & \textbf{96.48} & \textbf{94.75} & 95.11 & 96.30 & 96.62 & \textbf{96.68} & \textbf{96.67} \\
\textbf{RCS-SRMP} & ResNet-18 & 11.69 & 96.57 & {93.65} & \textbf{95.50} & \textbf{96.33} & \textbf{96.65} & 96.67 & \textbf{96.67} \\
\midrule
QCFS\cite{bu2023optimal} & ResNet-20 & 0.27 & 91.86 & 77.74 & 86.52 & 90.79 & 92.27 & 92.64 & \textbf{92.85} \\
SRP($\tau$=4)\cite{hao2023reducing} & ResNet-20 & 0.27 & 91.84 & $-$ & 90.51 & 91.37 & 91.64 & 91.72 & 91.80 \\
NQQCFS\cite{huang2025residual} & ResNet-20 & 0.27 & 85.18 & 77.30 & 84.54 & 87.43 & 88.26 & 88.51 & 88.45 \\
\textbf{RCS-SRMP} & ResNet-20 & 0.27 & 91.84 & \textbf{78.41} & \textbf{86.95} & \textbf{91.04} & \textbf{92.35} & \textbf{92.69} & 92.81\\
\midrule
\multicolumn{9}{c}{\textbf{CIFAR-100}} \\
\midrule
QCFS\cite{bu2023optimal} & VGG16 & 138.36 & 76.48 & 64.95 & 71.24 & 75.14 & 77.01 & 77.18 & 77.42 \\
SNM\cite{wang2022signed} & VGG16 & 138.36 & 76.53 & 65.31 & 72.00 & 76.04 & {77.23} & {77.37} & 77.17 \\
SRP($\tau$=4)\cite{hao2023reducing} & VGG16 & 138.36 & 76.48 & $-$ & 76.27 & 76.45 & 76.59 & 76.61 & 76.55 \\
NQQCFS\cite{huang2025residual} & VGG16 & 138.36 & 74.86 & \textbf{69.39} & \textbf{74.57} & 76.73 & \textbf{77.68} & \textbf{77.64} & \textbf{77.66} \\
\textbf{RCS-SRMP} & VGG16 & 138.36 & 76.45 & {65.72} & {72.69} & \textbf{76.96} & 77.10  & 77.32 & {77.45} \\
\midrule
QCFS\cite{bu2023optimal} & ResNet-18 & 11.69 & 78.69 & 69.36 & 75.45 & {78.55} & 79.49 & 79.60 & 79.54 \\
NQQCFS\cite{huang2025residual} & ResNet-18 & 11.69 & 76.66 & \textbf{70.86} & \textbf{76.81} & \textbf{78.85} & 79.44 & 79.62 & 79.46 \\
\textbf{RCS-SRMP} & ResNet-18 & 11.69 & 78.46 & {69.73} & {76.05} & 78.48 & \textbf{79.51} & \textbf{79.63} & \textbf{79.59}\\
\midrule
QCFS\cite{bu2023optimal} & ResNet-20 & 0.27 & 65.18 & 36.16 & 50.75 & 61.56 & \textbf{66.13} & \textbf{67.04} & 67.07 \\
SRP($\tau$=4)\cite{hao2023reducing} & ResNet-20 & 0.27 & 65.11 & $-$ & 59.34 & 62.94 & 64.71 & 65.50 & 65.82 \\
NQQCFS\cite{huang2025residual} & ResNet-20 & 0.27 & 62.34 & 33.87 & 50.28 & 60.93 & 64.73 & 65.43 & 65.35 \\
\textbf{RCS-SRMP} & ResNet-20 & 0.27 & 64.74 & \textbf{38.52} & \textbf{53.25} & \textbf{62.82} & 65.87 & 67.03 & \textbf{67.09} \\
\midrule
\multicolumn{9}{c}{\textbf{ImageNet}} \\
\midrule
QCFS\cite{bu2023optimal} & VGG16 & 138.36 & 71.29 & \text{1.48} & 4.71 & 19.07 & 50.92 & 68.33 & \textbf{70.88} \\
\textbf{RCS-SRMP} & VGG16 & 138.36 & 71.39 & 1.15 & \textbf{4.85} & \textbf{23.41} & \textbf{56.46} & \textbf{68.51} & 70.48 \\
\midrule
QCFS\cite{bu2023optimal} & ResNet-34 & 21.81 & 74.88 & 4.78 & 13.00 & 35.54 & 60.52 & \textbf{69.42} & \textbf{72.86} \\
\textbf{RCS-SRMP} & ResNet-34 & 21.81 & 72.94 & \textbf{6.16} & \textbf{14.38} & \textbf{37.90} & \textbf{60.74} & 69.01 & 72.48 \\
\bottomrule
\end{tabular}
\end{table*}

\subsection{Comparison to Other Conversion Methods on ViT Models}

Comparison with state-of-the-art spiking Transformer methods on ImageNet dataset are provided in Table \ref{tab:imagenet_compare}. $T_{ration}$ indicates the degree to which model accuracy degrades as the time step decreases, values approaching 1 correspond to higher model stability.  Our experiments on the ImageNet1k dataset outperforms most of the Transformer-SNN and ANN-SNN methods, and further push the frontiers of neural network accuracy without degrading efficiency. SRMP achieves an accuracy of $78.12\%$ at 6 time steps for ViT-B/16, higher than that of MST\cite{wang2023masked} at 128 time steps, and only $0.6\%$ lower compared with STA\cite{jiang2024spatio} at 32 time steps.

\subsection{Effect of SRMP on ViT under ECMT Framework}

We further evaluate the classification accuracy of fine-tuned ViT-S/16, ViT-B/16 and ViT-L/16 on CIFAR-10, CIFAR-100 and ImageNet datasets, equipped with the SRMP regularizer based on the ECMT framework\cite{huang2024highperformance}. The original ECMT framework adopts distinct positive and negative thresholds, yet this design yields no manifest accuracy improvement. For seamless integration with our fine-tuning pipeline, we unify the positive and negative thresholds without sacrificing classification performance. Consistent accuracy improvements are observed in T=2,4 compared with the vanilla ECMT model, as reported in Table~\ref{tab:vit_cifar10-100}. We achieve an accuracy of $7.19\%$ higher at T=2 for ViT-B/16 on ImageNet, and $1.75\%$ at T=4.

\begin{table}[!t]
\raggedbottom
\centering
\setlength{\tabcolsep}{5pt}
\caption{Comparison results on Transformer}
\label{tab:vit_cifar10-100}
\begin{tabular}{cccccc}
\toprule
\multirow{2}{*}{Arch.} & \multirow{2}{*}{Method} & 
\multirow{2}{*}{Param. (M)} &
\multirow{2}{*}{Original (ANN)} & \multicolumn{2}{c}{SNN} \\[-0.35ex]
\cmidrule(lr){5-6}
& & & & $T=2$ & $T=4$ \\
\midrule
\multicolumn{5}{c}{\textbf{CIFAR-10}} \\
\midrule
ViT-S/16 & ECMT\cite{huang2024highperformance} & 22 & 98.33 & 24.81 & 92.78 \\
ViT-S/16 & \textbf{SRMP} & 22 & 98.33 & \textbf{26.62} & \textbf{94.31} \\
ViT-B/16 & ECMT\cite{huang2024highperformance} & 86 & 98.73 & 85.67 & 96.76 \\
ViT-B/16 & \textbf{SRMP} & 86 & 98.73 & \textbf{87.17} & \textbf{97.03} \\
ViT-L/16 & ECMT\cite{huang2024highperformance} & 307 & 99.07 & 93.49 & 98.71 \\
ViT-L/16 & \textbf{SRMP} & 307 & 99.07 & \textbf{94.24} & \textbf{99.00} \\
\midrule
\multicolumn{5}{c}{\textbf{CIFAR-100}} \\
\midrule
ViT-S/16 & ECMT\cite{huang2024highperformance} & 22 & 89.28 & 10.88 & 77.03 \\
ViT-S/16 & \textbf{SRMP} & 22 & 89.28 & \textbf{16.53} & \textbf{78.56} \\
ViT-B/16 & ECMT\cite{huang2024highperformance} & 86 & 92.26 & 64.66 & 86.19 \\
ViT-B/16 & \textbf{SRMP} & 86 & 92.26 & \textbf{65.21} & \textbf{86.41} \\
ViT-L/16 & ECMT\cite{huang2024highperformance} & 307 & 93.84 & 69.62 & 89.65 \\
ViT-L/16 & \textbf{SRMP} & 307 & 93.84 & \textbf{71.12} & \textbf{90.13} \\
\midrule
\multicolumn{5}{c}{\textbf{ImageNet}} \\
\midrule
ViT-S/16 & ECMT\cite{huang2024highperformance} & 22 & 78.04 & 6.41 & 60.43 \\
ViT-S/16 & \textbf{SRMP} & 22 & 78.04 & \textbf{7.43} & \textbf{60.96} \\
ViT-B/16 & ECMT\cite{huang2024highperformance} & 86 & 80.77 & 12.33 & 67.84 \\
ViT-B/16 & \textbf{SRMP} & 86 & 80.77 & \textbf{19.52} & \textbf{69.59} \\
ViT-L/16 & ECMT\cite{huang2024highperformance} & 307 & 84.88 & 75.38 & 83.20 \\
ViT-L/16 & \textbf{SRMP} & 307 & 84.88 & \textbf{76.95} & \textbf{83.34} \\
\bottomrule
\end{tabular}
\vspace{-0.5em}
\end{table}

\begin{table*}[!t]
  \centering
  \caption{Comparison with previous works on ImageNet1k dataset}
  \footnotesize 
  \setlength{\tabcolsep}{11pt} 
  \begin{tabular}{l l c c c c c}
    \toprule
    Method & Conversion Type & Arch. & Param. (M) & T & $T_{ratio}$ & Acc(\%) \\
    \midrule
    \multirow{2}{*}{QCFS\cite{bu2023optimal}} & \multirow{2}{*}{ANN-to-SNN}
      & ResNet-34 & 21.8 & 64 & $-$ & 72.35 \\
    & & VGG-16 & 138 & 64 & $-$ & 72.85 \\
    \midrule
    \multirow{2}{*}{SRP\cite{hao2023reducing}} & \multirow{2}{*}{ANN-to-SNN}
      & ResNet-34 & 21.8 & 4(64) & 0.0625 & 66.71(68.61) \\
    & & VGG-16 & 138 & 4(64) & 0.0625 & 66.46(69.43) \\
    \midrule
    MST\cite{wang2023masked} & Transformer-to-SNN & Swin-T(BN) & 28.5 & 128(512) & 0.25 & 77.88(78.51) \\
    \midrule
    STA\cite{jiang2024spatio} & Transformer-to-SNN & ViT-B/32 & 86 & 32(256) & 0.125 & 78.72(82.79) \\
    \midrule
    \multirow{4}{*}{\textbf{ECMT-SRMP(Ours)}} & \multirow{4}{*}{Transformer-to-SNN}
      & ViT-S/16 & 22 & 8(10) & 0.8 & 76.03(77.07) \\
    & & ViT-B/16 & 86 & 6(10) & 0.6 & 78.12(80.59) \\
    & & ViT-L/16 & 307 & 4(8) & 0.5 & 83.34(84.71) \\
    \bottomrule
  \end{tabular}
  \label{tab:imagenet_compare}
\end{table*}

\subsection{Trade-off Between Conversion Error and Model Accuracy}
We recorded the variations of $Loss$ and $\mathcal{L}_\mathrm{RMPD}$ during the fine-tuning process, as shown in Fig. \ref{fig:Loss and LRMPD during fine-tuning}, A delicate balance is maintained between conversion error mitigation and classification accuracy loss during the tuning of $m_{\mathrm{init}}$.
From Eq.~\eqref{eq:total-loss}, we derive the gradient of total loss with respect to $m_{\mathrm{init}}$:
\begin{equation}
\frac{\partial \mathcal{L}_{\mathrm{total}}}{\partial m_{\mathrm{init}}} = \frac{\partial \mathrm{Loss}}{\partial m_{\mathrm{init}}} + \boldsymbol{\lambda} \cdot \frac{\partial \mathcal{L}_{\mathrm{RMPD}}}{\partial m_{\mathrm{init}}}.
\end{equation}

\begin{figure}[!t]
  \centering    \includegraphics[width=0.49\columnwidth]{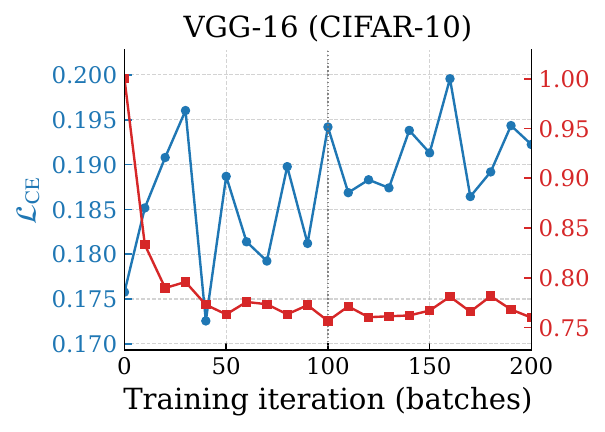}    \includegraphics[width=0.49\columnwidth]{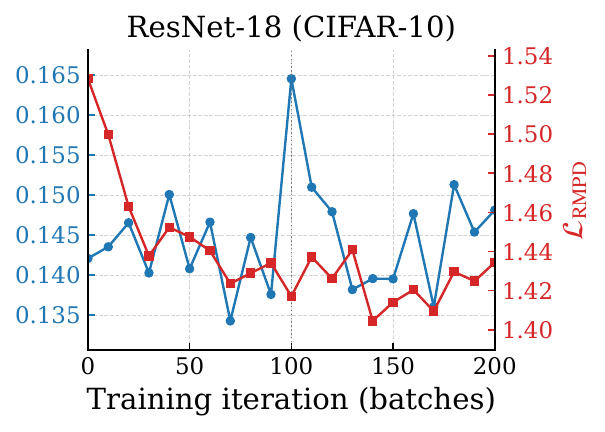}
    \hspace{0.2mm}     \includegraphics[width=0.485\columnwidth]{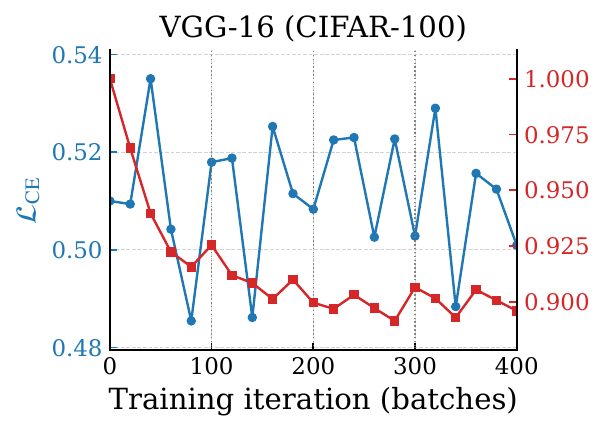}    \includegraphics[width=0.48\columnwidth]{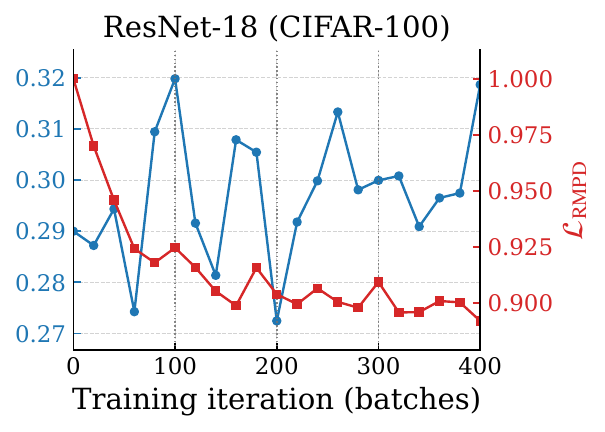}    \includegraphics[width=0.49\columnwidth]{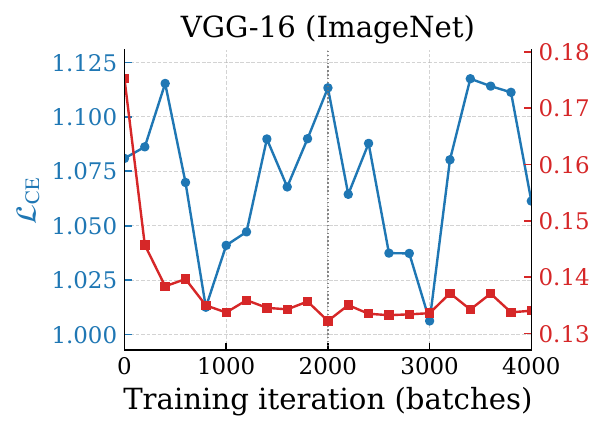}    \includegraphics[width=0.49\columnwidth]{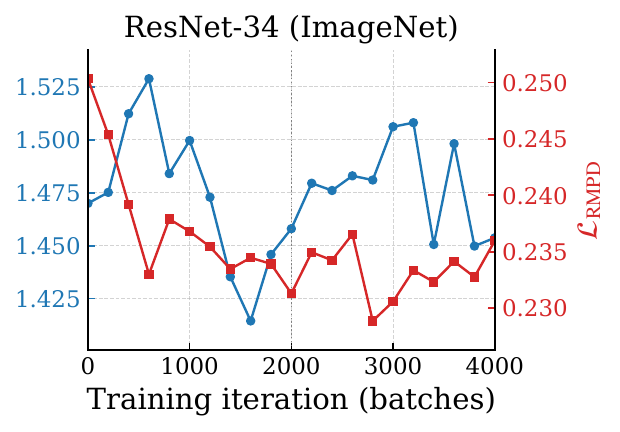}
\caption{Variations of $Loss$ and $\mathcal{L}_\mathrm{RMPD}$ during fine-tuning on CIFAR-10, CIFAR-100 and ImageNet.
$\mathcal{L}_\mathrm{RMPD}$ steadily decline, while $Loss$ fluctuate mildly or rise slightly.}
\label{fig:Loss and LRMPD during fine-tuning}
\end{figure}

This gradient is jointly governed by two components: the gradient of classification loss and the gradient of the distribution regularization loss, which balances the accuracy gain from optimized membrane potential distribution and the accuracy drop induced by ANN-SNN conversion.

Concretely, let $p$ denote the probability of spike mismatch caused by residual membrane potential falling inside the boundary region $B$. The theoretical expected conversion error is $\frac{p\theta}{L}\int_{B} f_U(x) dx$. As proven in Theorem 1, this conversion error is positively correlated with $\mathcal{L}_{\mathrm{RMPD}}$. Accordingly, $\frac{\partial \mathcal{L}_{\mathrm{RMPD}}}{\partial m_{\mathrm{init}}}$ quantifies how adjusting $m_{\mathrm{init}}$ impacts conversion errors, while $\frac{\partial \mathrm{Loss}}{\partial m_{\mathrm{init}}}$ captures the influence of $m_{\mathrm{init}}$ drift on classification performance.

At the early fine-tuning stage, $\int_{B} f_U(x) dx$ and $\mathcal{L}_{\mathrm{RMPD}}$ remain large, so the regularization loss dominates the update direction of $m_{\mathrm{init}}$. As training proceeds and $\mathcal{L}_{\mathrm{RMPD}}$ gradually decreases, the two gradient terms converge toward a relatively balanced state.

\subsection{Effect of Loss Induction on ViT}

We remove the Loss from $\mathcal{L}_{total}$
in Transformer to only record the variations of $\mathcal{L}_\mathrm{RMPD}$ for the ViT throughout fine-tuning. As shown in Fig. \ref{fig:li_loss_avg_vit_sbl_combined}, we evaluate the practical fine-tuning performance of ViT-S/16, ViT-B/16 and ViT-L/16 on CIFAR-10/100 and ImageNet. It can be observed that the average distribution shift gradually decreases and converges to a plateau during the fine-tuning process on most of the ViT models. Performance of ViT-L/16 on CIFAR-10 and CIFAR-100 exhibit a slow‑declining trend, yet it eventually stabilizes. The evident reduction of target interval shift after fine-tuning demonstrates the validity of $\mathcal{L}_\mathrm{RMPD}$, which further verifies its effectiveness in general activation-free ANN-SNN conversion. 

\begin{figure}[!t]
    \raggedbottom
    \centering
    \includegraphics[width=0.98\columnwidth]{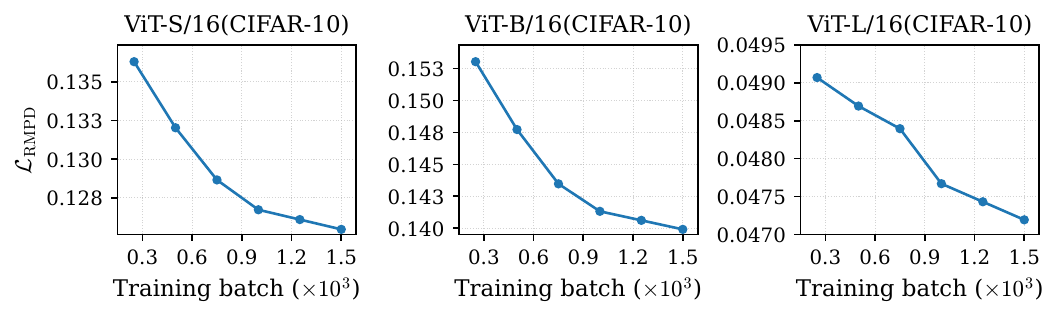}
    \includegraphics[width=0.98\columnwidth]{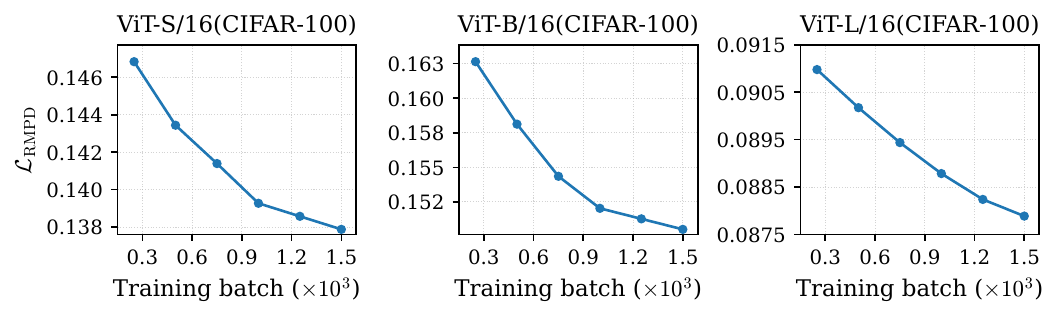}
    \includegraphics[width=0.98\columnwidth]{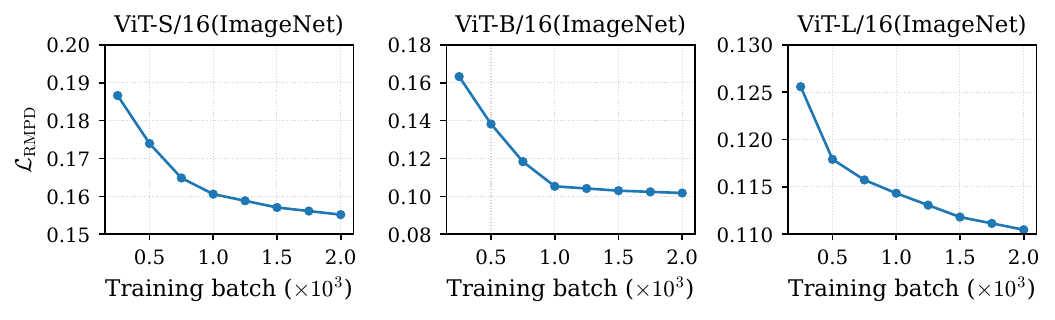}
    \caption{Curves of $\mathcal{L}_\mathrm{RMPD}$ induction in ViT-S/16, ViT-B/16 and ViT-L/16 model on CIFAR-10, CIFAR-100, ImageNet. ALL the curves decline steadily and flatten out over time. }
    \label{fig:li_loss_avg_vit_sbl_combined}
\end{figure}



\subsection{Conversion Accuracy Evaluation}

We first compare the classification accuracy of SCR-SRMP, SRMP, SCR and the baseline QCFS\cite{bu2023optimal} on CIFAR-10 dataset. Experimental results demonstrate that both SRMP and SCR-SRMP achieve consistent accuracy gains over vanilla QCFS, and they also outperform the naive ANN-SNN conversion scheme. SRMP features a concise and efficient pipeline; noticeable performance improvements can be observed after merely 10 fine-tuning epochs. The SCR-Conv2d module in SCR-SRMP delivers prominent performance gains under low time-step settings ($T=2,4$). Taking VGG-16 on the CIFAR-10 dataset as an example, SCR-SRMP achieves an accuracy of $92.84\%$ at $T=2$, substantially exceeding the accuracy of original QCFS. Detailed numerical results are summarized in Table~\ref{tab:cifar10_result}. As discussed in subsection \ref{subsec:lateral_inhibition_snn_training} and illustrated in the table, SCR-Conv2d primarily serves to enhance the fine‑tuning efficacy of SRMP; it yields no stable accuracy improvement when used alone. Therefore, we only evaluate SCR-SRMP in the subsequent full‑scale tests.

\begin{table}[!t]
  \raggedbottom
  \centering
  \setlength{\tabcolsep}{5.5pt}
  \caption{Comparison of test results among SRMP, SCR-SRMP, SCR and QCFS on CIFAR-10}
  \label{tab:cifar10_result}
  \begin{tabular}{lcccccc}
    \toprule
    \textbf{Method} & \textbf{Net.} & \textbf{ANN} & \textbf{2} & \textbf{4} & \textbf{8} & \textbf{16} \\
    \midrule
    \multicolumn{7}{c}{\textbf{CIFAR-10}} \\
    \midrule
    QCFS\cite{bu2023optimal} & VGG16 & 95.78 & 90.67 & 94.21 & 95.40 & 95.73 \\
    \textbf{SRMP} & VGG16 & 95.60 & {91.49} & \textbf{94.82} & \textbf{95.66} & \textbf{95.85} \\
    \textbf{SCR} & VGG16 & 95.65 & {91.92} & {94.36} & 95.40 & 95.77 \\
    \textbf{SCR-SRMP} & VGG16 & 95.65 & \textbf{92.84} & {94.80} & 95.46 & 95.78 \\
    \midrule
    QCFS\cite{bu2023optimal} & ResNet-18 & 96.48 & 92.86 & 94.92 & 96.21 & 96.55 \\
    \textbf{SRMP} & ResNet-18 & 96.57 & {93.34} & {95.20} & 96.02 & 96.51 \\
    \textbf{SCR} & ResNet-18 & 96.57 & {92.97} & {94.97} & 96.11 & 96.50 \\
    \textbf{SCR-SRMP} & ResNet-18 & 96.57 & \textbf{93.65} & \textbf{95.50} & \textbf{96.33} & \textbf{96.65} \\
    \midrule
    QCFS\cite{bu2023optimal} & ResNet-20 & 91.86 & 77.74 & 86.52 & 90.79 & 92.27 \\
    \textbf{SRMP} & ResNet-20 & 91.86 & {78.27} & {86.81} & 90.97 & 92.30 \\
    \textbf{SCR} & ResNet-20 & 91.84 & {78.31} & {86.51} & 90.77 & 92.23 \\
    \textbf{SCR-SRMP} & ResNet-20 & 91.84 & \textbf{78.41} & \textbf{86.95} & \textbf{91.04} & \textbf{92.35} \\
    \bottomrule
  \end{tabular}
\end{table}

\subsection{Ablation Study on Dynamic Regularization Weight}

In this section, we conduct ablation experiments to compare the performance of the dynamic regularization weight $\lambda_p$ and fixed constant regularization weight, and analyze the descending trend of loss term $\mathcal{L}_{\mathrm{RMPD}}$. Fig.\ref{fig:Ablation_six} reveals that $\mathcal{L}_{\mathrm{RMPD}}$ with fixed regularization weight achieves inferior convergence performance on CIFAR-10 compared to the dynamic weight $\lambda_p$. For CIFAR-100 and ImageNet, the fixed-weight scheme exhibits even worse convergence, or the loss $\mathcal{L}_{\mathrm{RMPD}}$ hardly decreases throughout training. It is worth noting that for ResNet-18 trained on CIFAR-10, the fixed-weight scheme yields a faster drop of $\mathcal{L}_{\mathrm{RMPD}}$. This stems from overfitting induced by the constant large regularization weight, which drastically degrades the original ANN’s accuracy. Consequently, fixed weight performs worse than the dynamic $\lambda_p$ in practice.

This phenomenon indicates that regularizer with a fixed weight fails to identify weight groups that require priority optimization, resulting in oscillating fine-tuning directions and negligible effective optimization progress. In contrast, the adaptive dynamic weight $\lambda_p$ strengthens inter-class contrast, thereby significantly boosting the overall convergence performance.

\begin{figure}[!t]
  \centering
    \includegraphics[width=0.98\columnwidth]{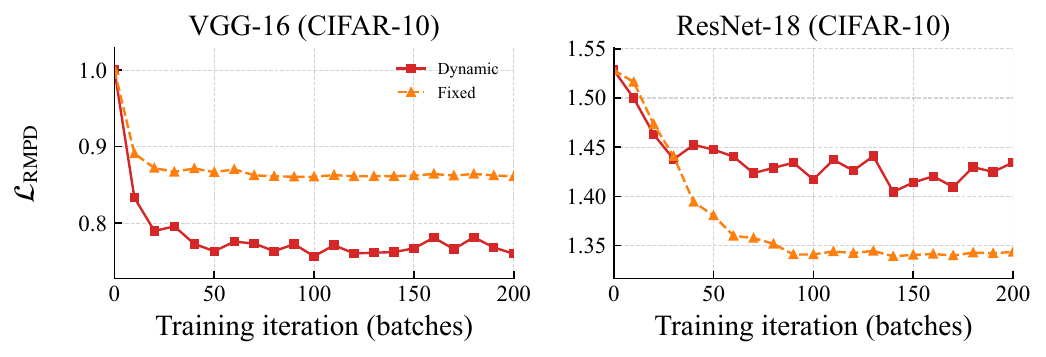}
    \includegraphics[width=0.98\columnwidth]{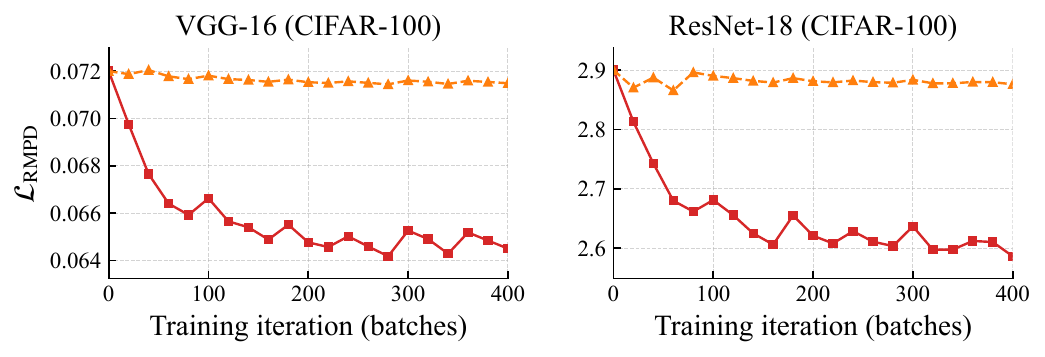}
    \includegraphics[width=0.98\columnwidth]{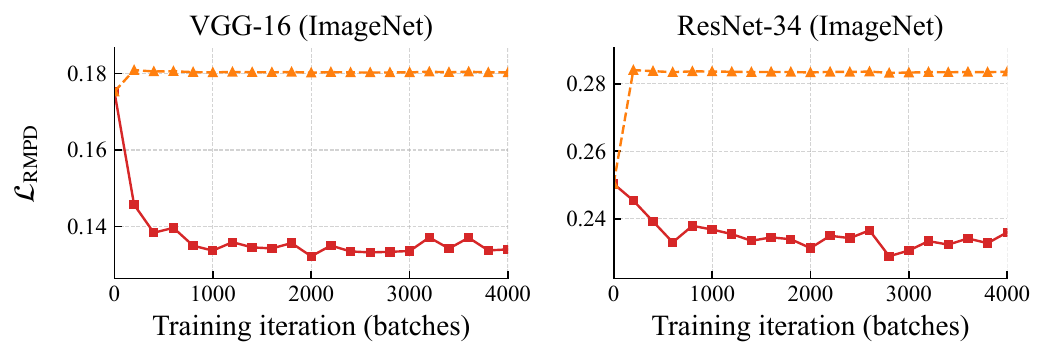}
\caption{Comparison with fixed weight and dynamic $\lambda_p$ for $\mathcal{L}_\mathrm{RMPD}$.}
\label{fig:Ablation_six}
\end{figure}

\subsection{Analysis of Input Mean Approximation}
\label{subsection:detail}

The class-wise input mean is generally approximated by averaging all samples belonging to the same class accumulated within a single batch or multiple consecutive batches. Suppose input samples follow the normal distribution $\Phi=\mathcal{N}(\mu, \sigma^2)$, with a total of $n$ samples and sample mean $\bar{I}$. The 95\% confidence interval of the true mean $\mu$ is formulated as:
\begin{equation}
\left[ \bar{I} - t_{0.025, n-1} \cdot \frac{\sigma}{\sqrt{n}},\ \bar{I} + t_{0.025, n-1} \cdot \frac{\sigma}{\sqrt{n}} \right].
\end{equation}
When $n\approx50$, the critical value satisfies $t_{0.025, n-1}\approx2$, which yields $t_{0.025, n-1} \cdot \frac{\sigma}{\sqrt{n}}\approx0.284\sigma$. This indicates the estimation error is less than $0.284\sigma$ with a confidence level of 95\%. Experimental verification demonstrates that $0.284\sigma<0.02$ for the vast majority of neurons, rendering this approximation sufficiently accurate. 
To demonstrate how many neurons satisfy this condition, we visualize the distribution of the input $\sigma$ of individual neurons in Fig. \ref{fig:vgg16_cifar10}. To intuitively characterize the magnitude of $\sigma$ via comparison with $\theta$, we plot the distribution of the ratio $\sigma/\theta$. It can be observed that the vast majority of $\sigma$ values remain small. Furthermore, the figure verifies the validity of the proposed theorem. 

\begin{figure}[!t]
  \raggedbottom
  \centering
  \begin{minipage}{0.48\columnwidth}
    \centering
    \includegraphics[width=\linewidth]{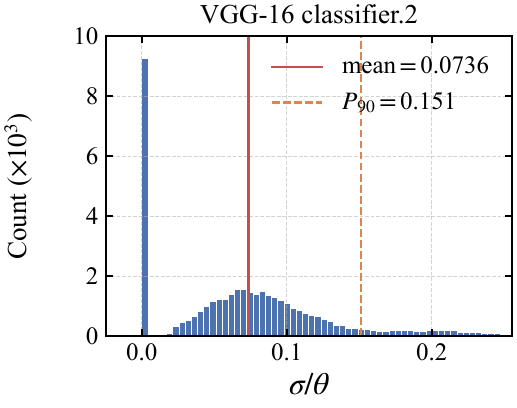}    
  \end{minipage}
  \begin{minipage}{0.48\columnwidth}
    \centering
    \includegraphics[width=\linewidth]{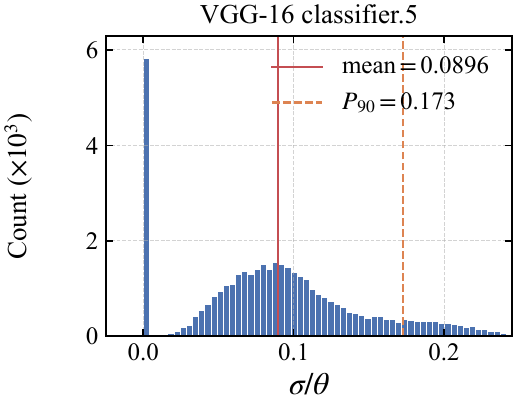}    
  \end{minipage}
  \begin{minipage}{0.48\columnwidth}
    \centering
    \includegraphics[width=\linewidth]{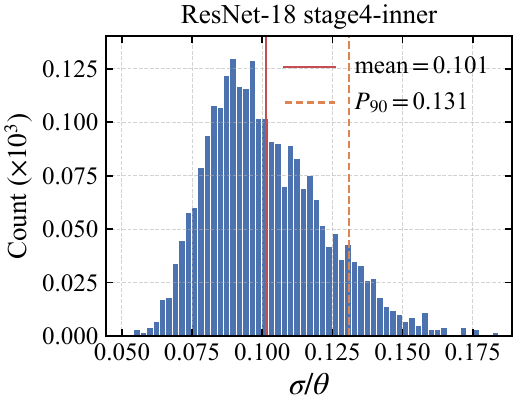}    
  \end{minipage}
  \begin{minipage}{0.48\columnwidth}
    \centering
    \includegraphics[width=\linewidth]{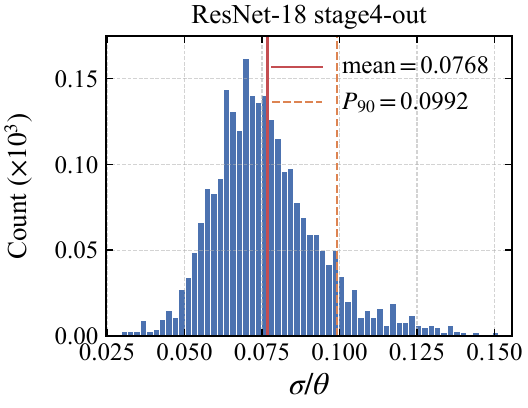}    
  \end{minipage}
  \begin{minipage}{0.48\columnwidth}
    \centering
    \includegraphics[width=\linewidth]{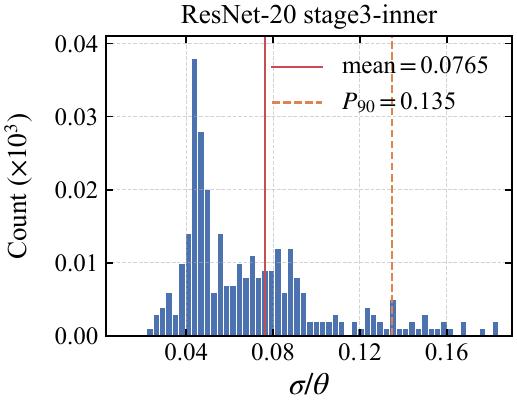}    
  \end{minipage}
  \begin{minipage}{0.48\columnwidth}
    \centering
    \includegraphics[width=\linewidth]{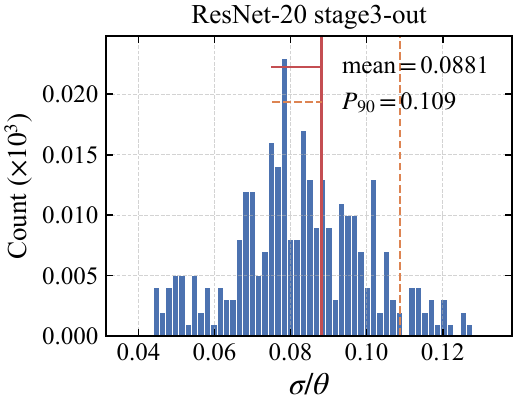}    
  \end{minipage}
  
  \caption{Distribution of $\sigma/\theta$ in SNN layers in VGG16, ResNet18 and ResNet-20. Frequency of $\sigma/\theta<0.125$ amounts to $84.67\%$ in VGG-16, which indicates that at least $84.67\%$ of neurons satisfy the conditions of Theorem \ref{thm:theorem_1}. Similar observations can be obtained for other models.}
  \label{fig:vgg16_cifar10}
\end{figure}

To guarantee a sufficiently large sample count $n$ per class, we introduce a hyperparameter $u_{\mathrm{in}}$ for cross-batch mean statistics when the batch size is small or the number of classes is large. The input mean is recalculated once every $u_{\mathrm{in}}$ batches; alternatively, an iterative update rule with historical mean records can be adopted to approximate the mean over large sample sets. For simplicity, we can constrain the number of samples from each class within a batch. For instance, given a batch size of 64 and 8 samples per class, each batch contains 8 distinct classes with 8 samples per class.

\subsection{Computational Complexity Comparison}
\label{sec:complexity_analysis}

 The comparison of the actual computational cost on each model compared with the original model without the SCR-Conv2d layer is presented in Table \ref{tab:li_computational_cost}. This result demonstrates that cost of model with SCR is nearly the same as base model, and that the computational overhead introduced by SCR-Conv2d is negligible compared with standard Conv2d. The theoretical analysis of computational cost is detailed in Appendix.

\begin{table}[!t]
\centering
\caption{Computational cost comparison.}
\label{tab:li_computational_cost}
\begin{tabular}{lllllc}
\toprule
Model & Base (M) & with CR (M) & $\Delta$C (M) & $\Delta$ (\%) & $\Delta$P. (K) \\
\midrule
VGG-16    & 666.50 & 671.36 & +4.86 & +0.73 & +24.2 \\
ResNet-18 & 1115.78 & 1120.50 & +4.72 & +0.42 & +4.0 \\
ResNet-20 & 83.24  & 84.12  & +0.88 & +1.06 & +0.4 \\
ResNet-34 & 2327.02& 2334.10& +7.08 & +0.30 & +6.9 \\
\bottomrule
\end{tabular}
\end{table}

\section{Conclusion}
In this work, we present SRMP, a general initial membrane potential tuning strategy for alleviating layer‑wise accumulated errors in ANN‑to‑SNN conversion. SRMP is valid for inputs with approximately normal distributions and yields considerable accuracy improvements. Spike competitive Refinement is adopted to strengthen model stability. As future work, reinforcement learning can be exploited to learn adaptive adjustment policies for SNN internal parameters. By extracting input‑aware latent features from spike‑rate states within SNNs, the learned policy can produce dynamic parameter adjustments.

\bibliographystyle{IEEEtran}  
\bibliography{ref}     

 




\newpage
\section*{Appendix}
\label{sec:Appendix}

\subsection{Proof of the Theorems}
\label{sec:proof_1_2}

\noindent \textbf{Theorem 1.}

(1) For an interval of fixed length defined over a normal distribution $\Phi$, the integral over this interval attains its maximum value if and only if the mean of the normal distribution coincides with the midpoint of the interval.

(2) Let the set of intervals be 
\begin{equation}
U = \biggl\{\biggl[0, \frac{\theta}{L}\biggr], \biggl[\frac{\theta}{L}, \frac{2\theta}{L}\biggr], \dots, \biggl[\frac{(L-1)\theta}{L}, \theta\biggr]\biggr\}. 
\end{equation}
Denote $U_T$ as the special case where exactly one interval satisfies the extremum condition stated in (1). 
The corresponding RMP distribution $F_{U_T}(x)$ is symmetric, and its gradient equals zero at both the midpoint and boundaries of each subinterval in $U_T$.

(3) When $L=4$ and $\frac{\theta}{L} > 2\sigma$, the boundary probability density of the distribution defined in (2) reaches its minimum. Formally, for any given $\epsilon>0$, define the boundary region $B=[0,\epsilon] \cup [\frac{\theta}{L}-\epsilon, \frac{\theta}{L}]$. The following relation holds:
\begin{equation}
\inf_{U}\int_{B} f_U(x) dx = \int_{B} f_{U_T}(x) dx,
\end{equation}
where $f_U(x)$ stands for the probability density function formed by superimposing normal distributions truncated on each subinterval of $U$. This indicates that $U_T$ yields the locally minimal probability density near interval boundaries.

In statements (1) and (2), we adopt the ideal assumption that the distribution strictly follows a complete normal distribution $\Phi=\mathcal{N}(\mu, \sigma^2)$. To better fit real input characteristics in statement (3), we approximate the distribution with symmetric truncated normal distributions, where the normalization constant $c$ satisfies $c \to 1$.

\begin{IEEEproof}
\label{thm:proof_1}

(1) Let the fixed interval length be $l$. Define the interval centered exactly at the distribution mean $\mu$ as $U=[\mu - \frac{l}{2}, \mu + \frac{l}{2}]$. Take an arbitrary shifted interval $U^\prime=[\mu - \frac{l}{2} + \delta, \mu + \frac{l}{2} + \delta]$ where the offset $\delta \neq 0$. For any $u \in U \setminus U^\prime$ and $v \in U^\prime \setminus U$, we have $\phi_{\mathcal{N}}(u) > \phi_{\mathcal{N}}(v)$, where $\phi_{\mathcal{N}}(\cdot)$ denotes the probability density function of the normal distribution.
The integral difference is formulated as:
\begin{equation}
\begin{split}
&\int_{U} \phi_{\mathcal{N}}(x) dx - \int_{U'} \phi_{\mathcal{N}}(x) dx\\
&= \int_{U \setminus U'} \phi_{\mathcal{N}}(x) dx - \int_{U' \setminus U} \phi_{\mathcal{N}}(x) dx > 0,
\end{split}
\end{equation}
which implies that the integral over the interval attains its maximum when the interval midpoint coincides with the distribution mean.

(2) Based on the prerequisite, the full range of residual membrane potential $[0, \theta]$ is evenly partitioned into $L$ subintervals:
\begin{equation}
\left[0, \frac{\theta}{L}\right], \left[\frac{\theta}{L}, \frac{2\theta}{L}\right], \dots, \left[\frac{(L-1)\theta}{L}, \theta\right].
\end{equation}
The target subinterval $I=\left[\mu - \frac{\theta}{2L}, \mu + \frac{\theta}{2L}\right]$ satisfies the extremum condition derived in (1), which guarantees all subintervals are symmetric about $x=\mu$. Owing to the symmetric property of normal distributions, the probability density function $f_U(x)$ of the residual membrane potential satisfies $f_U(\mu + x) = f_U(\mu - x)$, so the overall distribution $F_{U_T}(x)$ is symmetric.

At the central point $x=\mu$, the gradient of the distribution satisfies
\begin{equation}
\left.\frac{df_U(x)}{dx}\right|_{x=\mu} = \lim_{\Delta x \to 0} \frac{f_U(\mu+\Delta x) - f_U(\mu-\Delta x)}{\Delta x} = 0.
\end{equation}
At the boundaries $x=0$ and $x=\theta$, the membrane potential is truncated and confined within $0<u_t < \theta$, hence
\begin{equation}
\left.\frac{df_U(x)}{dx}\right|_{x=0} = \left.\frac{df_U(x)}{dx}\right|_{x=\theta} = 0.
\end{equation}
The derivation of statement (2) reveals the high similarity between the distribution of $U_T$ and the standard normal distribution.

(3) For simplified calculation, we set the normalization constant $c=1$.
For silent neurons that almost never fire spikes and saturated neurons that consistently emit $L=4$ spikes, their RMPDs correspond to shifted normal distributions $\mathcal{N}(\mu - k\theta, \sigma^2)$ with $k=0$ and $k=4$, respectively. Such neurons are excluded from subsequent analysis since their distributions are already purely Gaussian.
For general neurons, let $k$ denote the dominant spike count, and we approximate that only spike counts $\{k-1,k,k+1\}$ contribute noticeably to the distribution while other spike counts are negligible.

The full input distribution is thus covered by at most three subintervals:
\begin{equation}
\bigl[\tfrac{(k-2)\theta}{L}, \tfrac{(k-1)\theta}{L}\bigr],\;
\bigl[\tfrac{(k-1)\theta}{L}, \tfrac{k\theta}{L}\bigr],\;
\bigl[\tfrac{(k+1)\theta}{L}, \tfrac{(k+2)\theta}{L}\bigr],
\end{equation}
with the distribution mean satisfying $\mu \in \bigl[\tfrac{k\theta}{L}, \tfrac{(k+1)\theta}{L}\bigr]$.

The cumulative probability of residual membrane potential over the boundary region $B$ is calculated as
\begin{equation}
\begin{split}
\int_B f_U(x) dx 
&= \int_{0}^{\varepsilon} f_U(x) dx + \int_{\frac{\theta}{L} - \varepsilon}^{\frac{\theta}{L}} f_U(x) dx \\
&\approx \varepsilon \left( f_U(0) + f_U\left( \frac{\theta}{L} \right) \right) \\
&= 2\varepsilon \left( \phi\left( \frac{(k-1)\theta}{L} \right) + \phi\left( \frac{k\theta}{L} \right) \right).
\end{split}
\end{equation}

A well-known property states that the normal density function is concave outside the symmetric interval $[\mu - \epsilon, \mu + \epsilon]$. Under the condition $\frac{\theta}{L} > 2\sigma$, for any $\Delta x \in \left[0, \frac{\theta}{2L} - \sigma\right]$, we have
\begin{equation}
\begin{split}
&\phi\left(\mu - \frac{\theta}{2L} - \Delta x\right) + \phi\left(\mu - \frac{\theta}{2L} + \Delta x\right)\\
&\geq 2\phi\left(\mu - \frac{\theta}{2L}\right).
\end{split}
\end{equation}

For the interval configuration $U_T$ defined in (2), we have $\frac{(k-1)\theta}{L} = \mu - \frac{\theta}{2L}$ and $\frac{k\theta}{L} = \mu + \frac{\theta}{2L}$. Substitute the equalities into the density sum:
\begin{equation}
\begin{split}
&\phi\bigl(\tfrac{(k-1)\theta}{L}\bigr) + \phi\bigl(\tfrac{k\theta}{L}\bigr)\\
&= 2\phi\Bigl(\mu - \tfrac{\theta}{2L}\Bigr) \\
&\leq \phi\Bigl(\mu - \tfrac{\theta}{2L} - \Delta x\Bigr) + \phi\Bigl(\mu - \tfrac{\theta}{2L} + \Delta x\Bigr) \\
&= \phi\Bigl(\mu - \tfrac{\theta}{2L} - \Delta x\Bigr) + \phi\Bigl(\mu + \tfrac{\theta}{2L} - \Delta x\Bigr) \\
&= \phi\bigl(\tfrac{(k-1)\theta}{L} - \Delta x\bigr) + \phi\bigl(\tfrac{k\theta}{L} - \Delta x\bigr).
\end{split}
\end{equation}

This inequality indicates that shifting the interval set $U_T$ leftward by $\Delta x$ increases the value of $\int_{B} f_U(x) dx$. The same conclusion holds for a rightward shift by $\Delta x$. Therefore, the boundary integral $\int_{B} f_U(x) dx$ attains its local minimum exactly when $U=U_T$.
\end{IEEEproof}

\noindent \textbf{Theorem 2.}

For inputs following an approximate normal distribution, the cumulative probability over the target interval $[m-\frac{\theta}{2L},m+\frac{\theta}{2L}]$ is negatively correlated with the input variance and negatively correlated with the interval offset $|\mu-m|$. Here, $\theta$ denotes the membrane firing threshold, $\mu$ is the mean of the normal input distribution, and $m$ represents the midpoint of the target interval.

\begin{IEEEproof}
\label{thm:proof_2}
Let the input follow a normal distribution $\Phi=\mathcal{N}(\mu, \sigma^2)$. For the target interval $[m-\frac{\theta}{2L},m+\frac{\theta}{2L}]$, the cumulative probability over this interval is formulated as:
\begin{equation}
\int_{m-\frac{\theta}{2L}}^{m+\frac{\theta}{2L}} \frac{1}{\sqrt{2\pi}\sigma} \exp\left(-\frac{(x-\mu)^2}{2\sigma^2}\right) dx.
\end{equation}
Perform the variable substitution $t = \frac{x-\mu}{\sigma}$. The integral is transformed into the standard Gaussian integral form:
\begin{equation}
\int_{\frac{m-\mu-\frac{\theta}{2L}}{\sigma}}^{\frac{m-\mu+\frac{\theta}{2L}}{\sigma}} \frac{1}{\sqrt{2\pi}} \exp\left(-\frac{t^2}{2}\right) dt.
\end{equation}
As $\sigma$ increases, the absolute values of the integral bounds shrink, which reduces the value of the integral. Conversely, the integral grows as $\sigma$ decreases. Hence, the cumulative probability is negatively correlated with the variance $\sigma^2$.

Now consider two arbitrary intervals $U_0$ and $U_0'$. Analogous to the proof of statement (1) in Theorem 1, if $|\mu-m| < |\mu'-m|$, the integral of the probability density over $U_0$ is larger, where $m$ denotes the symmetric center of the target interval. Without loss of generality, assume $(\mu-m)(\mu'-m)\geq0$, meaning $U_0$ and $U_0'$ lie on the same side of $m$ (if not, reflect the interval across $x=\mu$ to satisfy this condition). For any $u \in U_0 \setminus U_0^\prime$ and $v \in U_0^\prime \setminus U_0$, we have $\phi_{\mathcal{N}}(u) > \phi_{\mathcal{N}}(v)$.

The difference between the two integrals reads:
\begin{equation}
\begin{split}
&\int_{U_0} \phi_{\mathcal{N}}(x)\,\mathrm{d}x - \int_{U_0^\prime} \phi_{\mathcal{N}}(x)\,\mathrm{d}x\\
&= \int_{U_0 \setminus U_0^\prime} \phi_{\mathcal{N}}(x)\,\mathrm{d}x - \int_{U_0^\prime \setminus U_0} \phi_{\mathcal{N}}(x)\,\mathrm{d}x > 0.
\end{split}
\end{equation}

This result demonstrates that the cumulative probability decreases monotonically as $|\mu-m|$ rises. Therefore, the cumulative probability is negatively correlated with the interval offset magnitude $|\mu-m|$.
\end{IEEEproof}

\subsection{Selection of Regularization Weights}
When applying SRMP fine-tuning for quantized conversion, we assign higher fine-tuning priority to classes with large input variance, and define the regularization weight as $\lambda_n = f(p)$ with $f' < 0$, where $p$ denotes the cumulative distribution density over the target interval.

However, for the general ANN-SNN conversion scheme discussed in Section V-A3, we place greater emphasis on the mean squared quantization error between quantized values and original continuous activations. Even inputs with small variance can incur substantial conversion error in this setting. We adopt piecewise mean squared error (MSE) defined as $E[(x-x_q)^2]$, where the quantized output is $x_q=\lfloor\frac{x+0.5}{\theta}\rfloor\cdot\theta$. The value $x_q$ remains constant within each partitioned interval; we omit the time-step hyperparameter $T$ for concise derivation. We analyze the two dominant practical scenarios: input concentrated within a single partition interval, and input distributed across two adjacent partition intervals (these two cases cover the vast majority of real-world inputs).

1. Single dominant interval: This corresponds to small interval offset. In this case, $\lfloor\frac{x+0.5}{\theta}\rfloor$ yields a unique integer value denoted as $n$. The MSE decomposition reads:
\begin{equation}
    E[(x-x_q)^2]=E[(x-\bar{x})^2]+E[(\bar{x}-x_q)^2]=\sigma^2+d^2,
\end{equation}
where $\sigma^2$ is the input variance and $d$ represents the interval offset distance. It follows that $E[(x-x_q)^2]$ is positively correlated with both $\sigma$ and $d$.

2. Two dominant intervals: This corresponds to large interval offset. By translation invariance, we simplify the analysis by setting the two adjacent intervals as $U=[-\theta, 0]$ and $V=[0,\theta]$, with positive mean offset $d=\bar{x}>0$. The MSE derivation expands as:
\begin{equation}
\begin{split}
    &E[(x-x_q)^2]\\
    &=E_{U}\left[\left(x+\frac{\theta}{2}\right)^2\right]+E_{V}\left[\left(x-\frac{\theta}{2}\right)^2\right]\\
    &=E[x^2]+\big[E_U(x)-E_V(x)\big]\cdot\theta \\
    &=E[x^2]-E(|x|)\cdot\theta \\
    &\geq D(x)+\bar{x}^2-\Big[E(|x-\bar{x}|)+\Phi(0<x<2\bar{x})\cdot\bar{x}\Big]\cdot\theta \\
    &=\sigma^2+\bar{x}^2-\left[\sqrt{\dfrac{2}{\pi}}\cdot \sigma+\Phi(0<x<2\bar{x})\cdot\bar{x}\right]\cdot\theta.
\end{split}
\end{equation}

The indicator term $\Phi(0<x<2\bar{x})$ is negatively correlated with $\sigma$. Under large offset conditions, $\Phi(0<x<2\bar{x})$ becomes sufficiently small and can be neglected. Taking the derivative with respect to $\sigma$ shows that for $\sigma<\sqrt{\frac{1}{2\pi}}\cdot\theta$ (a condition satisfied by nearly all input distributions), $E[(x-x_q)^2]$ is negatively correlated with $\sigma$.

In summary, for large interval offsets, inputs with smaller variance yield larger quantization errors — this conclusion contradicts the observation derived under the QCFS framework. Accordingly, we require $\lambda$ to be negatively correlated with input variance $\sigma$ and positively correlated with offset $d$. Drawing on Theorem~2, we can select the weight formulation $\lambda=f(d^k\cdot p)$ with $f'>0$, where a properly chosen constant $k$ ensures $\lambda$ is positively correlated with $d$.

\subsection{Theoretical analysis of Computational cost}
It is worth noting that the nonlinear activation inside SCR-Conv2d takes at most $T$ distinct values. In practical convolution computation, this operation is equivalent to performing convolution separately for $T$ groups of spike sequences with different magnitudes before aggregating all outputs. Accordingly, the number of floating-point multiplications introduced by SCR-Conv2d is at most $L$ times that of an identical network fed with uniform spike sequences. We compare the multiplication and addition FLOPs of standard Conv2d and the proposed SCR-Conv2d within the same layer as follows.

\textbf{Standard Conv2d}:

Suppose the input tensor shape is $[N, C_{\mathrm{in}}, H, W]$, the output tensor shape is $[N, C_{\mathrm{out}}, H, W]$, and the convolution kernel shape is $[C_{\mathrm{out}}, C_{\mathrm{in}}, k, k]$. For simplified analysis, we assume both layers receive spike outputs from the preceding layer with identical firing rate $f_{\mathrm{s}}$. The total number of addition FLOPs for Conv2d is
\begin{equation}
\text{FLOP}_{\mathrm{add}} = N \cdot H \cdot W \cdot C_{\mathrm{out}} \cdot (k^2 \cdot C_{\mathrm{in}}) \cdot f_{\mathrm{s}},
\end{equation}
and the total number of multiplication FLOPs is
\begin{equation}
\text{FLOP}_{\mathrm{mul}} = N \cdot C_{\mathrm{out}} \cdot C_{\mathrm{in}} \cdot k^2.
\end{equation}
The overall computational complexity is formulated as:
\begin{equation}
C = N \cdot C_{\mathrm{out}} \cdot C_{\mathrm{in}} \cdot k^2 \cdot \left( H \cdot W \cdot f_{\mathrm{s}} \cdot C_{\mathrm{add}} + C_{\mathrm{mul}} \right).
\end{equation}

\textbf{SCR-Conv2d}:

The activation tensor $\mathit{act}$ has shape $[N, C, H, W]$, and the weight tensor $W_{\mathrm{SCR}}$ follows the shape $[C, 1, k, k]$. The convolution implementation of SCR-Conv2d is realized via
$\texttt{F.conv2d(act, W, \text{groups}=x.\text{shape}[1])}$.

The total addition FLOPs of SCR-Conv2d are
\begin{equation}
\text{FLOP}_{\mathrm{add}}^{\mathrm{SCR}} = N \cdot H \cdot W \cdot C_{\mathrm{out}} \cdot k^2 \cdot f_{\mathrm{s}}.
\end{equation}

The upper bound of multiplication FLOPs is
\begin{equation}
\text{FLOP}_{\mathrm{mul}}^{\mathrm{SCR}} = N \cdot C_{\mathrm{out}} \cdot k^2.
\end{equation}
The overall computational complexity is given by
\begin{equation}
C_{\mathrm{SCR}} = N \cdot C_{\mathrm{out}} \cdot k^2 \cdot \left( H \cdot W \cdot f_{\mathrm{s}} \cdot C_{\mathrm{add}} + 5 \cdot C_{\mathrm{mul}} \right).
\end{equation}

We derive the complexity ratio as

\begin{equation}
\frac{C_{\mathrm{SCR}}}{C} = \frac{N C_{\mathrm{out}} k^2 \left( H W f_{\mathrm{s}} C_{\mathrm{add}} + 5 C_{\mathrm{mul}} \right)}{N C_{\mathrm{out}} C_{\mathrm{in}} k^2 \left( H W f_{\mathrm{s}} C_{\mathrm{add}} + C_{\mathrm{mul}} \right)} \approx \frac{1}{C_{\mathrm{in}}}.
\label{eq:li_ratio}
\end{equation}

This result demonstrates that the computational overhead introduced by SCR-Conv2d is negligible compared with standard Conv2d. Furthermore, competitive refinement improves spike sparsity across neurons, such that the practical running cost of the model equipped with SCR-Conv2d is roughly equivalent to the baseline network.

\vfill

\end{document}